\documentclass[conference]{IEEEtran}

\usepackage{graphicx}

\usepackage{soul}

\usepackage{amsmath}
\usepackage{amsthm}
\usepackage{amsfonts}

\usepackage{mathtools}

\usepackage[dvipsnames]{xcolor}

\usepackage{wrapfig}
\usepackage{comment}

\usepackage[linesnumbered,ruled,vlined]{algorithm2e}

\usepackage{siunitx}

\usepackage{multirow}

\usepackage{mathabx}

\usepackage{wrapfig}
\usepackage{subcaption}
\usepackage{setspace}

\usepackage{amsmath,amsthm}

\usepackage{times}

\usepackage[numbers]{natbib}
\usepackage{multicol}
\usepackage[bookmarks=true]{hyperref}

\newcommand{\tnm}[1]{{\textbf{\texttt{#1}}}} % Name of module/tool

\usepackage{amssymb,amsmath,amsthm}

\usepackage[capitalise,english,nameinlink]{cleveref}

\usepackage[inline]{enumitem}

\newtheorem{theorem}{Theorem}
\newtheorem{problem}{Problem}
\newtheorem{solution}{Solution}
\newtheorem{assumption}{Assumption}
\newtheorem{remark}{Remark}

\newcommand{\reals}{\mathbb{R}}

\newcommand{\integers}{\mathbb{Z}}

\newcommand{\argmin}[1]{\underset{#1}{\operatorname{arg}\,\operatorname{min}}\;}

\newcommand{\visavis}{\textit{vis-\`a-vis}}

\newcommand{\perVec}{\mathbf{s}}
\newcommand{\perVecChoice}[1]{\hat{\mathbf{s}}_{#1}}
\newcommand{\inpVec}{\mathbf{u}}
\newcommand{\inpVecOpt}{\mathbf{u}^*}
\newcommand{\inpVecChoice}[1]{\hat{\mathbf{u}}_{#1}}
\newcommand{\choiceVec}{\mathbf{c}}
\newcommand{\binChoiceVec}{\mathbf{b}}
\newcommand{\per}[1]{\hat{s}_{#1}}
\newcommand{\loss}[1]{\mathcal{L}(#1)}
\newcommand{\perModel}[1]{g_{\text{per}}(#1)}
\newcommand{\perModelCost}[1]{cost_{\text{per}}(#1)}

\newcommand{\controlCost}[1]{\mathcal{J}^{\text{ctrl}}(#1)}
\newcommand{\perceptionCost}[1]{\mathcal{J}^{\text{per}}(#1)}
\newcommand{\totalCost}[1]{\mathcal{J}^{\text{tot}}(#1)}

\newcommand{\expectation}[1]{\mathbb{E}\left[#1\right]}

\newcommand{\variance}[1]{Var\left(#1\right)}

\newcommand{\perModelExp}[1]{g_{\text{exp}}(#1)}

\newcommand{\perModelVar}[1]{g_{\text{var}}(#1)}

\newcommand{\perModelExpCont}[1]{g^{c}_{\text{exp}}(#1)}

\newcommand{\perModelVarCont}[1]{g^{c}_{\text{var}}(#1)}

\newcommand{\perModelCostCont}[1]{cost^c_{\text{per}}(#1)}

\title{\LARGE \bf Dynamic Selection of Perception Models for Robotic Control}

\author{\IEEEauthorblockN{Bineet Ghosh\IEEEauthorrefmark{1},
Masaad Khan\IEEEauthorrefmark{2}, Adithya Ashok\IEEEauthorrefmark{2},
Sandeep Chinchali\IEEEauthorrefmark{2},
Parasara Sridhar Duggirala\IEEEauthorrefmark{1}\\
}
\IEEEauthorblockA{\IEEEauthorrefmark{1}The University of North Carolina at Chapel Hill, USA \\
\IEEEauthorrefmark{2}The University of Texas at Austin, USA \\
Email: \texttt{\{bineet,psd\}@cs.unc.edu, \{mak4668,adithyashok,sandeepc\}@utexas.edu }
}
}

\begin{document}
\maketitle

\begin{abstract}
Robotic perception models, such as Deep Neural Networks (DNNs), are becoming more computationally intensive and there are several models being trained with accuracy and latency trade-offs. 
However, modern latency accuracy trade-offs largely report mean accuracy for single-step vision tasks, but there is little work showing which model to invoke for multi-step control tasks in robotics. 
The key challenge in a multi-step decision making is to make use of the right models at right times to accomplish the given task.
That is, the accomplishment of the task with a minimum control cost and minimum perception time is a desideratum; this is known as the model selection problem.
In this work, we precisely address this problem of invoking the correct sequence of perception models for multi-step control.
In other words, we provide a provably optimal solution to the model selection problem by casting it as a multi-objective optimization problem balancing the control cost and perception time.
The key insight obtained from our solution is how the variance of the perception models matters (not just the mean accuracy) for multi-step decision making, and to show how to use diverse perception models as a primitive for energy-efficient robotics. 
Further, we demonstrate our approach on a photo-realistic drone landing simulation using visual navigation in AirSim. Using our proposed policy, we achieved 38.04\% lower control cost with 79.1\% less perception time than other competing benchmarks.
% Modern day robotic systems have access to a plethora of models--with varying compute cost and accuracy, such as EfficientNets and ResNets--to accomplish a desired control task with minimum control cost. The key challenge in such a setting is to make use of the right models at right times such that the task is accomplished, with a minimum control cost, in minimum compute time---this is known as the \emph{model selection problem}. In this paper, we provide a provably optimal solution to the model selection problem by casting it as an optimization problem.
% %
% Modern latency accuracy trade-offs largely look at single step vision tasks and report mean accuracy. However, in this paper, we show that how variance matters for multi-step decision making and show how to use diverse models as a primitive for energy-efficient robotics.
% %
% Further, we demonstrate our approach on an aircraft landing simulation. Using our policy, we were able to land the aircraft in reasonable amount of time, without causing driver discomfort. Other policies, on the other hand, either causes a rapid, uncomfortable landing, or takes unreasonably long time to land the aircraft. 
\end{abstract}

\IEEEpeerreviewmaketitle

%%%%%%%%%%%%%%%%%%%%%%%%%%%%%%%%%%%%%%%%%%%%%%%%%%%%%%%%%%%%%%%%
\section{Introduction}
\label{sec:introduction}
%\Bineet{Comments.}

%\Sandeep{Comments.}

%\Sridhar{Comments.}

% A robot, required to perform a desired control task, is generally equipped with a low accuracy and computationally fast model (local model). However, in modern day distributed robotic setting, in addition to the local model, the robot has access to a plethora of models with varying cost and accuracy, such as EfficientNets~\cite{effnet} and ResNets~\cite{resnets}. In such a setting---a robot equipped with a plethora of models (\textit{i.e.} can make invocation) with varying cost and accuracy---the key goal of the robot is to accomplish the task with a required control cost in minimum compute time. In other words, the goal is to achieve optimality---a perfect balance between compute time and control cost. 
A resource constrained robot, required to perform a desired control task, is generally equipped with a low accuracy and computationally fast perception model (local model). 
However, today different perception models with accuracy and latency trade-offs are being trained.
For example, the EfficientNet~\cite{effnet} suite of vision models provides 8 model variants, namely EfficientNet 0 to EfficientNet 7, that trade-off accuracy with model size and latency (ResNet~\cite{resnets} suite is yet another such example).
These models increase inference latency and energy consumption, requiring a robot to gracefully balance control cost with perception cost, as shown in \cref{fig:effnetTrained}.
The key desideratum for the robot in a multi-step control task is to leverage the available models to perform the task efficiently. 
In other words, the robot should accomplish the task by optimally balancing the perception time and control cost. 
In the rest of the paper, we refer to perception time as perception cost since it can represent various metrics (such as number of bit-operations performed, time taken to run on a device, \textit{etc}.).

\begin{figure}[t]
\vskip 0.22in
\begin{center}
{\includegraphics[width=1.0\columnwidth]{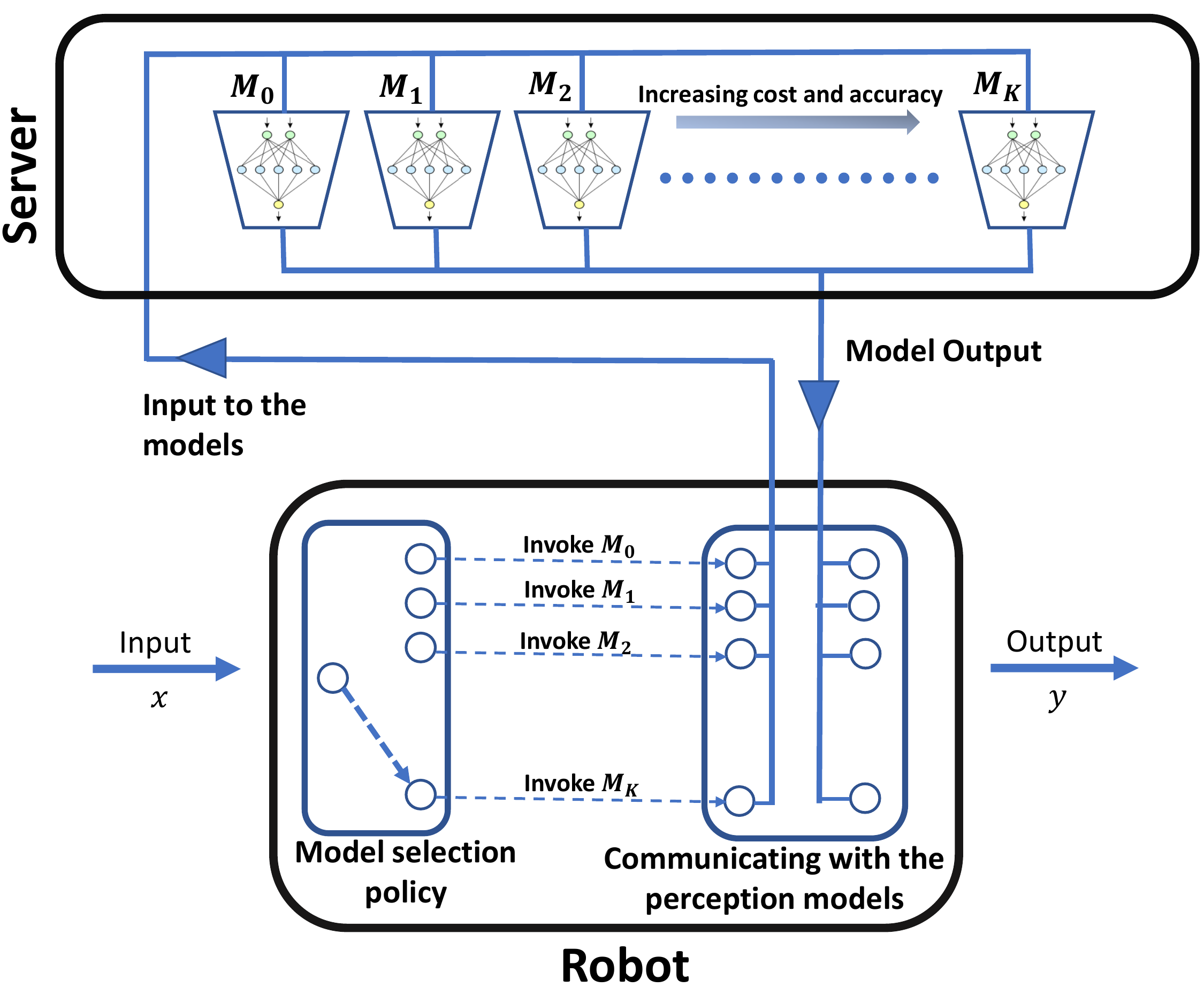}}
\end{center}
%\vskip -0.2in
\caption{\textbf{Model Selection Problem: }
In a setting where a robot has access to a plethora of models (namely $M_0$ to $M_K$)---with varying cost and accuracy---the goal of the robot is to perform a control task such that it strikes an optimal balance between compute cost and control cost.}
\label{fig:one}
%\vskip -0.25in
\end{figure}

With recent advancements in model training (with accuracy and latency trade-off), a model selection strategy that can accomplish a task by optimally balancing the perception cost and control cost is a necessity in model day robotic applications.
Imagine a scenario where a factory robot is required to perform safe navigation in a busy warehouse floor. 
%The robot is equipped with a plethora of perception models that trades-off control cost and perception cost, to perform a safe navigation. 
To perform the navigation, the robot is equipped with a plethora of perception models that trades-off control cost and accuracy.
Without loss of generality, such perception models can reside on the robot, or in a central server (``Cloud") which can be invoked over a shared network. In such scenarios, a desired strategy should be to use the high perception cost models, to ensure minimum control cost, \emph{only} when the robot is required to make tricky maneuvers. 
The solution we propose in this paper provably realizes such a policy. Consider a solution using the computationally fastest perception model among all the available models (say EfficientNet 0) at all time steps: this will lead to a solution that takes minimum possible perception cost but results in worst possible control cost (as decisions based on a low-accuracy model might lead to states far away from the optimal states \visavis{} control cost). Similarly, a solution using the most computationally expensive perception model (say EfficientNet 7) can result in a least control cost solution but maximum perception cost. Clearly, for any practical scenario, such simple solutions do not work. A desired solution, on the other hand, uses computationally expensive perception models only when computationally faster models fail to produce the desired result. Adding to the challenges, the relationship, with respect to compute cost and accuracy, between the available plethora of models is highly nonlinear~\cite{effnet}. We refer to this as the \emph{model selection problem}, illustrated in \cref{fig:one}.

\begin{figure}[t]
\vskip 0.22in
\begin{center}
{\includegraphics[width=1.0\columnwidth]{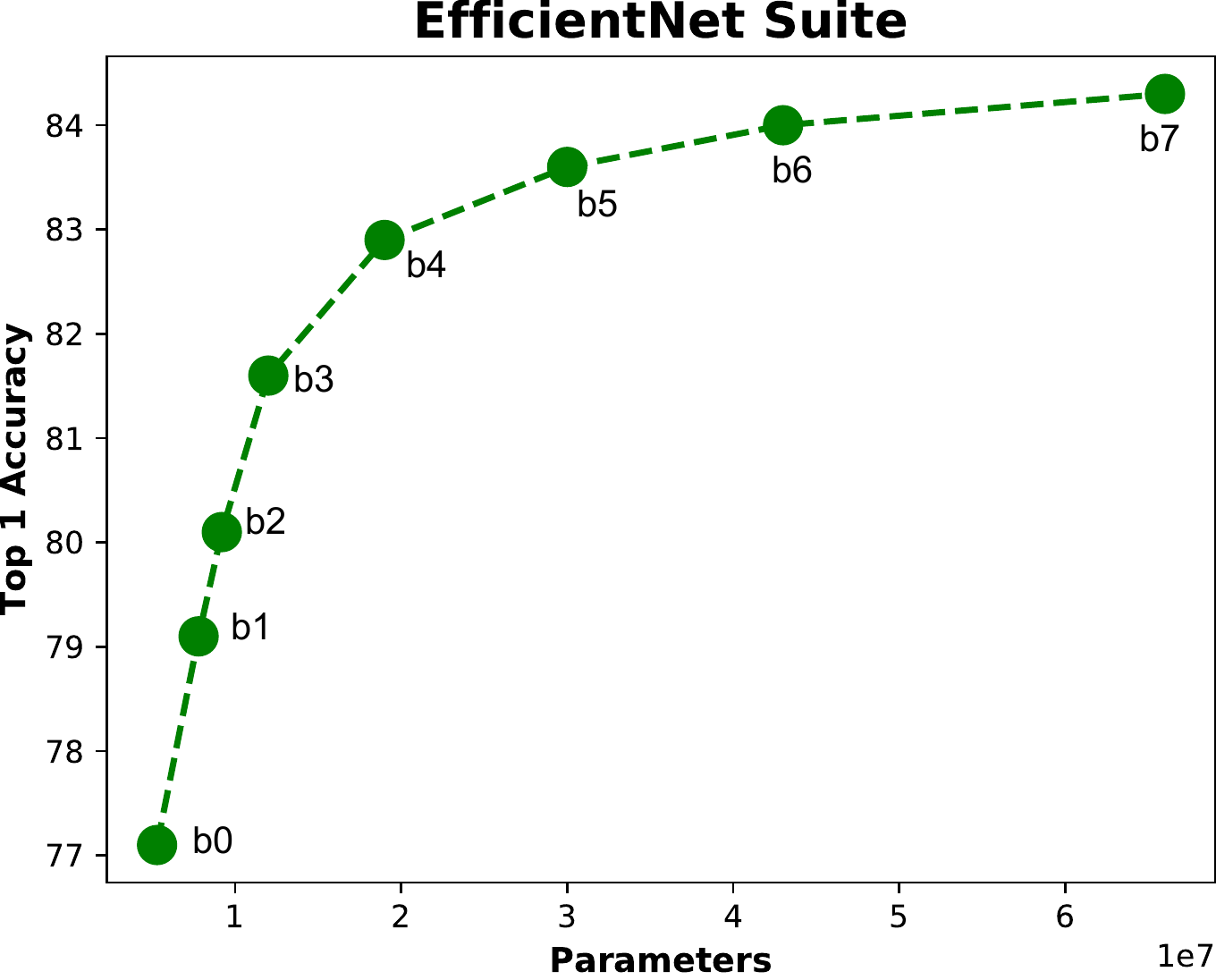}}
\end{center}
%\vskip -0.2in
\caption{\textbf{EfficientNet Suite.} (The data is taken from~\cite{effnet}).
What model should we dynamically invoke based on a robot's real-time operating context?}
\label{fig:effnetTrained}
%\vskip -0.25in
\end{figure}

% An earlier work has provided an optimal solution to this problem, with two models, and only when certain assumptions hold true  \cite{ghoshiros21}. In this work,  we provide a provably optimal solution to the model selection problem with a plethora of models (any number of models) with varying cost and accuracy, even when the prior assumptions does not hold true. In particular, we compute a sequence of models that should be invoked, at each time step, to achieve optimality with respect to compute time and control cost. The key insight behind our solution is: the model selection sequence is computed by formulating it as an optimization problem (MIQP) that minimizes a linear combination of compute cost and control cost, leveraging results from~\cite{sandeep_neurips}. Further, we prove a polynomial time complexity for our proposed solution for special subcases. We also note that, though we motivate the paper with the goal of minimizing compute time and maximizing accuracy, but our proposed framework can be used with other metrics as well---for instance, one can use safety in place of control cost, or network bandwidth instead of compute time. 
% %
% Our solution also indicates the importance of variance (not just mean accuracy) of the models while considering a multi-step decision problem. Whereas the model latency accuracy trade-offs largely look at the single step vision tasks and report mean accuracy. In particular, our works shows that the multi-step solution depends linearly on variance as well.

Most optimal control only penalizes control cost and tracking error, while remaining oblivious to the perception cost.
The key idea behind our solution is to compute the sensitivity of control cost to perception errors, which in turn guides which model to invoke since we not only care about conventional control cost but also perception latency.
In particular, we provide a provably optimal solution to the model selection problem with a plethora of perception models (any number of models) with varying control cost and perception cost.
That is, we compute a sequence of models that should be invoked, at each time step, to achieve optimality with respect to control cost and perception cost. 
This is achieved by encoding the model selection sequence as a Mixed Integer Quadratic Program (MIQP) that minimizes a linear combination of the control cost and the perception cost, leveraging results from~\cite{sandeep_neurips}. Intuitively, the integer variables, in the MIQP formulation, model the discrete available choices of the perception models, and the objective function is quadratic primarily due to the control cost being a quadratic function of the robot states. 
Further, we prove a polynomial time complexity for our proposed solution for special subcases. 
%
%We also note that, though we motivate the paper with the goal of minimizing compute time and maximizing accuracy, but our proposed framework can be used with other metrics as well. For instance, one can use safety in place of control cost, or network bandwidth instead of compute time.
%
Our solution also indicates the importance of variance (not just mean accuracy) of the perception models when considering a multi-step decision problem; where in contrast, most works on model latency accuracy trade-offs largely look at the single step vision tasks and report mean accuracy.

% We demonstrate our solution on an aircraft landing sequence simulated on AirSim~\cite{airsim}. In our experiments, an aircraft is equipped with all the EfficientNet models to compute the altitude of the aircraft. The aircraft landing controller issues input to the elevator based on the altitude computed by the given EfficientNet model (where the model, at a given time step, is selected as per the policy). Our experiments have shown the following interesting insights:
% \begin{enumerate}
%     \item Using EfficientNet 0, at all time steps, results in a quick landing of the aircraft, but a jerky and uncomfortable landing. Clearly this is an unacceptable policy.
%     \item Using EfficientNet 7, at all time steps, results in the most comfortable landing, but takes a huge amount of time to land the aircraft. This strategy is not feasible in practice---an aircraft cannot be allowed to take a very long time to make the landing.
%     \item Whereas our model selection strategy---that uses all the available models efficiently---results in a comfortable landing in a reasonable amount of time.
% \end{enumerate}

%We demonstrate our solution on an photo-realistic drone landing sequence, using visual navigation, simulated in AirSim~\cite{airsim}. 
We demonstrate our solution on a drone landing sequence, using visual navigation, simulated on a photo-realistic simulator named AirSim~\cite{airsim}.
In our experiments, the drone is equipped with all the EfficientNet models to compute the altitude of the drone from the helipad. 
These models have associated perception errors (in estimating the altitude) due to uncertainty in DNNs.
The drone landing controller issues input based on the altitude computed by the given EfficientNet model (where the model, at a given time step, is selected as per the policy). Our experiments have shown that we can achieve 38.04\% lower control cost in 79.1\% less perception time than other available policies. For the rest of the paper, we refer to perception models simply as models.

\textit{Related Work}: The closest works to our work are \cite{chinchali2019network,ghoshiros21}. In \cite{chinchali2019network}, given two models, with low inference cost (fast) and high inference cost but higher accuracy (slow), it computes a model selection sequence using a Reinforcement Learning (RL) based approach. However, \cite{chinchali2019network} does not guarantee an optimal solution. In \cite{ghoshiros21}, when the input to the models are independent and identically distributed, using statistical correlation between the two models (fast and slow), it provides a provably optimal model selection sequence. 
The key difference of our work with \cite{ghoshiros21} is that their model selection approach is primarily focused on inference whereas our work focuses on control tasks.
The approach mentioned in \cite{ghoshiros21} primarily depends on model compression based techniques \cite{mit_pac_1,mit_pac_2} to compute the statistical correlation between the models. Further, our work can also be used in scenarios where a robot must selectively ask a human teacher during active learning \cite{cakmak2012designing,whitney2017reducing} or for manipulation \cite{kaipa2016enhancing}, when the mean accuracy and variance of the human teacher is known.
While not directly studying the model selection problem, \cite{nakanoya2021co} proposes an algorithm to compress sensory data that is relevant to the perception model's objective.

\textit{Contributions}: Given prior works, mainly \cite{chinchali2019network, ghoshiros21}, our contributions are as follows:
\begin{enumerate}
    \item Prior works have only considered the model selection problem with only two models. \cite{chinchali2019network} uses an RL based approach, and \cite{ghoshiros21} provides an interpretable solution when some restrictive assumptions hold true. In contrast, we provide a provable optimal solution to the model selection problem, with any number of perception models, for multistep decision making to achieve a desired a control task.
    \item \cite{chinchali2019network} uses RL based approach, and \cite{ghoshiros21} uses statistical correlation between the models primarily using model compression. In contrast, our solution is fundamentally different from both these approaches. Leveraging recent results from \cite{sandeep_neurips}, we cast the model selection problem as an optimization problem (MIQP), and provide a provable optimal solution.
    \item Our solution indicates the importance of variance, not just mean accuracy of the models. Further, our work also allows the user to use a plethora of compute models, such as EfficientNets \cite{effnet} and ResNets \cite{resnets}, to achieve a goal in an optimal fashion.
    \item We show that a special sub-case of the model selection problem has polynomial time complexity.
    \item We illustrate our approach on a photo-realistic drone landing sequence simulated in AirSim. The results of our experiments clearly show how our approach outperforms other policies.
\end{enumerate}

\textit{Organization}: The paper is organized as follows. In \cref{sec:probStatement}, we define the problem of model selection formally. In \cref{sec:all_aware_solution_model_selection}, we provide an optimal solution when the output of all the models, at all time steps, is known. In \cref{sec:exp_solution_model_selection}, we provide an optimal solution when the expected accuracy and variance of the models is known. \cref{sec:polytime_subcase} shows a special sub-case of the model selection problem that can be solved in polynomial time. Finally in \cref{sec:exp}, we illustrate our approach on a drone landing sequence simulated in AirSim.

% \section{Notations}
% \label{sec:notations}
% %\input{notations}

\section{Problem Statement}
\label{sec:probStatement}
In this section, we define the problem statement formally.

Consider the following dynamics:
\begin{equation}
\label{eq:dynamics}
    x_{t+1}=A x_t + Bu_t + Cs_t,
\end{equation}
where, at time step $t$, $x_t$ denotes the state reached by the robot, $u_t$ denotes the controller input to the robot, and $s_t$ denotes the \textit{perception input}. We note that, in the dynamics given in \cref{eq:dynamics}, the state of the robot $x_t$ is therefore dependent on the perception inputs to the robot at previous time steps.
Though a linear dynamics, as in \cref{eq:dynamics}, is restrictive, but it provides a good first step especially with linearized robot dynamics for common robots. Further, the additive $C s_t$ term can seem restrictive but indeed practical, especially while encoding exogenous sensor measurements. Imagine a drone landing sequence that relies on a DNN based perception model to compute the altitude of the drone which is used by the controller to issue control inputs to the drone. The DNN estimates the altitude of the drone with some error due to the inherent uncertainties of the model. In such cases, the $C s_t$ term is used to model the measurement errors from the DNN. 

\textit{\underline{Robot State}}: Given the dynamics as in \cref{eq:dynamics}, the set of states reached by the robot from an initial set $x_0 \in \mathbb{R}^n$ at time step $t$ is given by $x_t \in \mathbb{R}^n$, where $A \in \mathbb{R}^{n \times n}$.

\textit{\underline{Perception Input}}: The perception input to the robot at time step $t$ is given by $s_t \in \mathbb{R}^p$, where $C \in \mathbb{R}^{n \times p}$. Let $\mathbf{s} \in \mathbb{R}^{pH \times 1}$ denote the vector of all perception inputs up-to time step $H-1$; that is:
\begin{equation*}
    \perVec =
    \begin{bmatrix}
    s_0 & s_1 & \dots & s_{H-1}
    \end{bmatrix}^\top.
\end{equation*}

\textit{\underline{Controller Input}}: Given a perception input vector $\perVec$, the controller input to the robot at time step $t$ is given by $u_t=\pi\big(x_t,s_{t:H-1},\theta_c\big)$, where $\pi$ is the policy providing the control inputs, with the given parameter $\theta_c$ (such as dynamics matrix, feedback matrix \textit{etc.})~\cite{sandeep_neurips}.  Let $\inpVec=\mu(\perVec)$ denote the vector of all control inputs, for a given perception input vector $\mathbf{s}$, up-to time step $H-1$; that is:
\begin{equation*}
    \inpVec = \mu(\perVec) =
    \begin{bmatrix}
    u_0 & u_1 & \dots & u_{H-1}
    \end{bmatrix}^\top,
\end{equation*}

where $u_t=\pi\big(x_t,s_{t:H-1},\theta_c\big)$. We note that the control input $u_t$ is dependent on the current state ($x_t$) and the future perception inputs ($s_{t:H-1}=\{s_t, s_{t+1}, \cdots, s_{H-1}\}$).
To be able to compute the control inputs that would make the system stable, it needs the information about what it sees as perception input in the future as well. The closed form expression for this can be found in~\cite{sandeep_neurips}.

\subsection{Perception Model}
\label{subsec:perception_model}

The true perception input to the robot, at time step $t$, is given as $s_t$. But obtaining the true perception input $s_t$ is practically impossible, therefore we use practically implementable perception models to estimate $s_t$ as $\per{t}$. The loss between the true perception input $s_t$ and the estimated perception input is given by $\loss{s_t,\per{t}}$.

The perception input, at time step $t$, to the robot comes from the perception model $g_{\text{per}}: \integers_{[0,W-1]} \times \integers_{\ge 0} \rightarrow \mathbb{R}^{p}$, where $W$ is the total number of perception models available. 
Intuitively, the perception model $\perModel{w,t}$ selects one of the available models $0 \le w \le W-1$, and returns its estimated perception at time step $t$.
The cost incurred to the robot for invoking the perception model is given by $cost_{\text{per}}: \integers_{[0,W-1]} \rightarrow \mathbb{R}$. We note that our cost is dependent only on the model selected, and not on time step. That is, the perception input at time step $t$ is given by $\per{t}=\perModel{w,t}$ for a given \textit{choice of the model} $w$, and the cost for invoking the model is given by $\perModelCost{w}$. Intuitively, the model choice $w$ represents the desired quality of output by the user. For example, higher value $w$ incurs a higher cost \big(\textit{i.e.} higher value of $\perModelCost{w}$\big) and better quality of perception input \big(\textit{i.e.} lower value of $\loss{\perModel{w,t},s_t}$\big).

We note that $\per{t}$ is used to represent the perception input at time step $t$; where $\per{t}=\perModel{w,t}$ means the perception input is chosen from the $w^{\text{th}}$ perception model as per the function $\perModel{w,t}$. However, the function $\perModel{\cdot}$, for all $t$ and all $w \in \integers_{[0,W-1]}$, is defined independent of $s_t$ as follows:

\begin{equation}
    \label{eq:perceptionModel}
    \perModel{w,t}= \Gamma^w_t,
\end{equation}
\begin{equation}
    \label{eq:perceptionCost}
    \perModelCost{w}=w \cdot \Upsilon,
\end{equation}

where $w \in \integers_{[0,W-1]}$ represents the model choice variable, $\Upsilon \in \reals_{>0}$, and $\Gamma^w_t \in \mathbb{R}^p$. Intuitively, $\Gamma^0_t$ represents the lowest-quality perception input \big(\textit{i.e.} highest value of $\loss{\Gamma^0_t,s_t}$\big) with minimum perception cost 0 (at time step $t$), and $\Gamma^{W-1}_t$ represents the highest-quality perception input with maximum perception cost $\perModelCost{W-1}=(W-1)\cdot \Upsilon$ (at time step $t$). 

Let $\mathbf{C}$ be the set of all possible choices. Let $\choiceVec=\big\{w_0, w_1, \cdots, w_{H-1}\big\} \in \mathbf{C}$ be a set of choices up-to time step $H-1$. 
Intuitively, $\mathbf{C}$ represents the decision variables in the optimization problem encoding the choice of the models at every time step. For instance, an optimal solution $\choiceVec_{\text{opt}} \in \mathbf{C}$ to the optimization problem with $w_t=M$ implies model $M$ (where $0 \le M \le W-1$) should be invoked at time step $t$.
Let the perception input vector, corresponding to $\choiceVec$, be given as follows:
\begin{equation*}
    \perVecChoice{\choiceVec} =
    \begin{bmatrix}
    s^c_0 & s^c_1 & \dots & s^c_{H-1}
    \end{bmatrix}^\top,
\end{equation*}
where $s^c_t=\perModel{w_t,t}$, and the corresponding control vector is given as $\inpVecChoice{\choiceVec}=\mu(\mathbf{s}_{\mathbf{c}})$. The optimal control vector $\inpVecOpt=\mu(\perVec)$.

\subsection{Perception Cost and Control Cost}
\label{subsec:perception_control_costs}

Given a set of choices $\choiceVec=\big\{w_0, w_1, \cdots, w_{H-1}\big\}$, and an initial set $x_0$, the control cost $\mathcal{J}^{ctrl}(\mathbf{c})$ is given as \cite{sandeep_neurips}:

\begin{align}
\label{eq:jc}
    \controlCost{\choiceVec}=J^c(\inpVecChoice{\choiceVec},\perVec,x_0) = \nonumber \\
    c_{H-1}(x_{H-1}) + \sum\limits_{t=0}^{H-1} c(x_t, \hat{u}_t),
\end{align}

where $c(\cdot)$ is the stage cost and $c_{H-1}(x_{H-1})$ is the terminal cost. Further, we define the perception cost $\mathcal{J}^{per}(\mathbf{c})$ as follows:

\begin{equation}
\label{eq:jp}
    \perceptionCost{\choiceVec}=\sum\limits_{t=0}^{H-1} \perModelCost{w_t}.
\end{equation}

Given $\alpha, \beta \in \mathbb{R}$, the total cost $\mathcal{J}^{\text{tot}}$ gracefully trades-off the control cost and the perception cost by taking a weighted sum of the two, with parameters $\alpha$ (denoting the weight on the control cost) and $\beta$ (denoting the weight on the perception cost). Formally:
\begin{equation}
\label{eq:jtot}
   \totalCost{\choiceVec} =\alpha \cdot \controlCost{\choiceVec} + \beta \cdot \perceptionCost{\choiceVec}.
\end{equation}

\subsection{Model Selection Problem}
\label{subsec:model_selection_problem}
Using the aforementioned artifacts, we now formally state the model selection problem as follows:
\begin{problem}[Model Selection Problem]
\label{prob:model_selection_naive}
Given a system with dynamics as in \cref{eq:dynamics}, and perception model $\perModel{\cdot}$ with cost $\perModelCost{\cdot}$, compute a set of choices $\choiceVec_{opt}$ for $H$ time steps, such that, for a given $\alpha,\beta$ and $x_0 \in \mathbb{R}^n$:
\begin{equation}
    \choiceVec_{\text{opt}}= \argmin{\choiceVec \in \mathbf{C}} \Big\{ \totalCost{\choiceVec} \Big\}.
\end{equation}
\end{problem}

Intuitively, we want to find a set of choices, $\choiceVec_{\text{opt}}$, such that the total cost $\mathcal{J}^{\text{tot}}$ (an weighted sum of the control cost and the perception cost) is minimized.

\subsubsection{Reformulating \cref{prob:model_selection_naive}}
In this section we reformulate \cref{prob:model_selection_naive} using results from \cite{sandeep_neurips}. The reason to reformulate \cref{prob:model_selection_naive} is three-fold: \begin{enumerate*}
    \item In \cref{prob:model_selection_naive}, the solution is tied to an initial set $x_0$. In this reformulation, we propose a formulation that is independent of the initial state.
    \item Leveraging prior results, we will propose a solution that can be formulated as an optimization problem. 
    %Therefore, it can be solved very efficiently using off the shelf tools.
    Such a formulation helps us in computing the optimal model selection sequence efficiently using the off the self optimization solvers.
    \item We will further show that for special cases we can solve the problem in polynomial time complexity.
\end{enumerate*} 

A re-formulation to solve \cref{prob:model_selection} can be given as follows:
\begin{align}
&~&
\begin{aligned}
\min_{\choiceVec \in \mathbf{C}} \quad \totalCost{\choiceVec}
\end{aligned} \label{opt:naive}\\
&\iff&
\begin{aligned}
\min_{\choiceVec \in \mathbf{C}} \quad \alpha \cdot \controlCost{\choiceVec} + \beta \cdot \perceptionCost{\choiceVec}
\end{aligned} \label{opt:naive2} \\
&\iff&
\begin{aligned}
\min_{\choiceVec \in \mathbf{C}} \quad \underbrace{\alpha \cdot \big(J^c(\inpVecChoice{\choiceVec},\perVec,x_0)\big)}_{\text{control term}} + \underbrace{\beta \cdot \perceptionCost{\choiceVec}}_{\text{perception cost term}}
\end{aligned} \label{opt:naive3}
\end{align}

The control term of the optimization (in \cref{opt:naive3}) is minimized with control inputs computed from the true perception input $\perVec$. That is, the control cost increases when estimated perception inputs $\perVecChoice{}$ is used instead of true perception input $\perVec$. Formally, for any set of choices $\choiceVec \in \mathbf{C}$, the following relation is assumed to hold:
\begin{align}
    J^c\big(\inpVecOpt,\perVec,x_0\big) < J^c(\inpVecChoice{\choiceVec},\perVec,x_0). \label{eq:relCtrlCost}
\end{align}

Therefore, the re-formulation given in \cref{opt:naive2} can be further rewritten as: 
\begin{equation}
\label{opt:model_selection}
\begin{aligned}
\min_{\choiceVec \in \mathbf{C}} \quad \alpha \cdot  \Big(J^c(\inpVecChoice{\choiceVec},\perVec,x_0) - J^c\big(\inpVecOpt,\perVec,x_0\big) \Big) + \beta \cdot \perceptionCost{\choiceVec}
\end{aligned}
\end{equation}

Given an arbitrary perception input vector $\hat{\mathbf{s}}$, \cite{sandeep_neurips} provides a closed form solution to $J^c\big(\mu(\hat{\mathbf{s}}),\perVec,x_0\big)$ $-$ $J^c\big(\inpVecOpt,\perVec,x_0\big)$. More specially, \cite{sandeep_neurips} provides a closed form expression for $J^c\big(\mu(\hat{\mathbf{s}}),\perVec,x_0\big) - J^c\big(\inpVecOpt,\perVec,x_0\big)$  in terms of $\hat{\mathbf{s}}$ and $\perVec$ (as a quadratic function of the difference between $\hat{\mathbf{s}}$ and $\perVec$, \textit{i.e.}, a quadratic function of $\hat{\mathbf{s}}-\perVec$):
\begin{equation}
    \label{eq:control_cost_diff}
    J^c\big(\mu(\hat{\mathbf{s}}),\perVec,x_0\big) - J^c\big(\inpVecOpt,\perVec,x_0\big) = (\hat{\mathbf{s}} - \perVec)^{\top} \times \Psi \times (\hat{\mathbf{s}} - \perVec),
\end{equation}
where $\Psi=L^{\top}K^{-1}L$ is a positive semi-definite matrix. \cite{sandeep_neurips} provides closed form solutions to compute the matrices $L$,$K$, where $L$ and $K$ are only dependent on the dynamics (matrices $A$, $B$ and $C$) and cost matrices $Q$ and $R$. We note that the above term (\cref{eq:control_cost_diff}) is independent of $x_0$, using this we propose a reformulation of \cref{prob:model_selection} that is independent of the initial set $x_0$ in the rest this section.

Given a set of choices, $\choiceVec \in \mathcal{C}$, recall that the perception input vector corresponding to $\choiceVec$ is $\perVecChoice{\choiceVec}$. Using \cref{eq:control_cost_diff}, the optimization formulation in \cref{opt:model_selection} can be equivalently rewritten as:

\begin{problem}[Model Selection Problem]
\label{prob:model_selection}
Given a system with dynamics as in \cref{eq:dynamics}, and perception model $\perModel{\cdot}$ with cost $\perModelCost{\cdot}$, compute a set of choices $\choiceVec_{opt}$ for $H$ time steps, such that, for a given $\alpha,\beta$:
\begin{equation}
\label{opt:model_selection_main}
\begin{aligned}
\min_{\choiceVec \in \mathbf{C}} \quad \alpha \cdot  \Big((\perVecChoice{\choiceVec} - \perVec)^{\top} \times \Psi \times (\perVecChoice{\choiceVec} - \perVec) \Big) + \beta \cdot \perceptionCost{\choiceVec}.
\end{aligned}
\end{equation}
\end{problem}

\begin{theorem}[Equivalence]
\label{thm:equivalence}
\cref{prob:model_selection_naive} and \cref{prob:model_selection} are equivalent.
\end{theorem}
\begin{proof}
In \cref{appx:correctness}.
\end{proof}

\textit{We note that the reformulation in \cref{prob:model_selection} is independent of the initial set $x_0$}. That is, the set of optimal choices $\choiceVec_{opt}$, as computed by our solution, is independent of $x_0$. %Therefore, the model selection problem need not be solved from scratch when the initial set changes, if our solution is used.
It is worthwhile to mention that a model selection sequence, that is independent of the initial state, enables an engineer/practitioner to compute the model selection sequence just once and use it anywhere regardless of its initial set. In most practical scenarios, where the initial state cannot be known apriori, our solution is extremely effective as the model selection sequence need not be solved from scratch when the initial state changes.

\section{Model Selection Problem When All the Perception Outputs are Known}
\label{sec:all_aware_solution_model_selection}
We note that \cref{prob:model_selection} assumes that the difference between the perception model's output $\perModel{\cdot}$ and the true perception is known. But in practice, this might not hold true for many cases. In such cases, one can only know a distribution, not the exact difference between the true and perception models' output. In this section, we will propose an exact solution to \cref{prob:model_selection} (assuming the true difference is known), but in subsequent sections we will relax this assumption to handle practical scenarios.

Since our solution to the problem is based on an optimization formulation, before we put forward our solution to \cref{prob:model_selection}, we wish to layout our model selection encoding. That is, how the perception $\per{t}$ and the cost $cost_t$, at a time step $t$, can be encoded using binary variables:

\begin{align}
    \per{t} = \sum\limits_{i=0}^{W-1} b^t_i \cdot \perModel{i,t} \label{eq:perEncoding},\\
    cost_t = \sum\limits_{i=0}^{W-1} b_i^t \cdot \Upsilon \label{eq:costEncoding},
\end{align}

such that $\sum\limits_{i=0}^{W-1} b^t_i = 1$, where $\forall_{b_i^t \in \integers_{[0,1]}}$. Intuitively, if $b_w^t=1$ (and rest all 0), it implies that the model $0 \le w \le W-1$ has been selected at time step $t$; \textit{i.e.}, $\per{t}=\perModel{w,t}$ with $cost_t=\perModelCost{w}$. 

Let $\binChoiceVec$ denote the set of Boolean variables required to encode the model choices up-to time step $H$.
\begin{equation*}
    \binChoiceVec =
    \{
    b^0_0, b^0_1, \dots, b^0_{W-1}, b^1_{0}, \dots b^1_{W-1}, \dots b^{H-1}_{W-1}.
    \}
\end{equation*}

Note that, clearly, $\choiceVec$ and $\binChoiceVec$ are equivalent (\cref{appx:bc_equivalence}).

Using this idea, one can encode $\perVecChoice{c}$ as follows (simply by replacing the elements $\per{t}$ in $\perVecChoice{c}$):
\begin{equation}
\label{eq:s_c}
\perVecChoice{c}
=
\begin{bmatrix}
\sum\limits_{i=0}^{W-1} b^0_i \cdot \perModel{i,0} \\
\vdots \\
\sum\limits_{i=0}^{W-1} b^t_i \cdot \perModel{i,t}\\
\vdots \\
\sum\limits_{i=0}^{W-1} b^{H-1}_i \cdot \perModel{i,H-1}
\end{bmatrix}.
\end{equation}

Therefore, $\perVecChoice{c} - \perVec$ can encoded and rewritten, in terms of $\binChoiceVec$, as follows: 
\begin{equation}
\label{eq:perDiffEncoding}
\perVecChoice{c} - \perVec
=
\mathcal{V}(\binChoiceVec)=
\begin{bmatrix}
\sum\limits_{i=0}^{W-1} b^0_i \cdot \perModel{i,0} - s_0\\
\vdots \\
\sum\limits_{i=0}^{W-1} b^t_i \cdot \perModel{i,t} - s_t\\
\vdots \\
\sum\limits_{i=0}^{W-1} b^{H-1}_i \cdot \perModel{i,H-1} - s_{H-1}
\end{bmatrix},
\end{equation}

such that $\forall_{0 \le t \le H-1} \sum\limits_{i=0}^{W-1} b_i^t=1$.

Similarly, $\perceptionCost{\choiceVec}$ can be encoded and overloaded with $\binChoiceVec$ as follows:

\begin{equation}
\label{eq:costBinEncoding}
\perceptionCost{\choiceVec}
=
\perceptionCost{\binChoiceVec}
=
\sum\limits_{t=0}^{H-1} \left( \sum\limits_{i=0}^{W-1} b_i^t \cdot i \cdot \Upsilon \right).
\end{equation}

Therefore, using \cref{eq:perDiffEncoding} and \cref{eq:costBinEncoding}, we provide an optimal solution to \cref{prob:model_selection} in \cref{sol:model_selection} using a Boolean optimization problem.

\begin{solution}[Solution to \cref{prob:model_selection}]
\label{sol:model_selection}
The following Boolean optimization problem provides an optimal solution to \cref{prob:model_selection}.
\begin{equation}
\begin{aligned}
\min_{\binChoiceVec} \quad & \alpha \cdot  \Big(\mathcal{V}(\binChoiceVec)^{\top} \times \Psi \times \mathcal{V}(\binChoiceVec) \Big) + \beta \cdot \perceptionCost{\binChoiceVec} \\
\textrm{s.t.} \quad & \forall_{t} \sum\limits_{i=0}^{W-1} b^t_i = 1\\
  &\forall_{t,i} b^t_i \in \integers_{[0,1]}.    \\
\end{aligned}
\end{equation}
\end{solution}

\begin{remark}
The Boolean optimization formulation in \cref{sol:model_selection}, using $H \cdot W$ Boolean variables, is NP-Hard. \cite{qp_complexity,qp_complexity_2}
\end{remark}

We note that though the proposed solution is in general NP-Hard, but for most practical cases, using off the shelf tools, one can solve such optimization problems very efficiently using several heuristics. In particular, optimization based solutions have been effectively used in embedded hybrid model predictive control~\cite{miqp_app1}, optimal coordination of vehicles at intersections~\cite{miqp_app2} \textit{etc}. The same will also been witnessed in our experiments.

\section{An expectation based approach to the Model Selection Problem}
\label{sec:exp_solution_model_selection}
In \cref{sec:all_aware_solution_model_selection} we proposed a solution to \cref{prob:model_selection} assuming that $(\perVecChoice{c} - \perVec)$ is known. However, in reality, this assumption might not always hold true, for reasons such as unavailability of perception models \big($\perModel{\cdot}$ in \cref{eq:perceptionModel}\big) or true perception input $\perVec$. Whereas, a probabilistic distribution of $(\perVecChoice{c} - \perVec)$ can still be known \cite{mit_pac_1,mit_pac_2,pac_lr}. In other words, instead of various perception models \big(as in \cref{eq:perceptionModel}\big), we have various distributions available to us at varying model costs. Formally, we assume the knowledge of the following:

\begin{equation}
    \label{eq:perceptionModelDistro}
    \left(\perModel{w,t} - s_t\right) \thicksim \mathcal{T}^w_t,
\end{equation}
\begin{equation}
    \label{eq:perceptionCostDistro}
    \perModelCost{w}=w \cdot \Upsilon,
\end{equation}

where $w \in \integers_{[0,W-1]}$ represents the model choice variable, $\Upsilon \in \reals_{>0}$, and $\mathcal{T}^w_t$ is a probability distribution represented with the random variable $\mathcal{D}^w_t$. Intuitively, $\mathcal{D}^0_t$ represents the lowest-quality perception input distribution (\textit{i.e.} high variance and mean farther away from 0) with minimum perception cost $\perModelCost{c_t}=0$ (at time step $t$). $\mathcal{D}^{W-1}_t$ represents the highest-quality perception input distribution (\textit{i.e.} low variance and mean closer to 0) with maximum perception cost $\perModelCost{W-1}=(W-1)\cdot \Upsilon$ (at time step $t$). Note that we do not impose any assumption on the distributions. 
However, in practice, the distributions are generally uniform or normal as the models estimate the target variable fairly well for most of the inputs, while performing relatively poor for some outlier cases. 
Given a random variable $\mathcal{D}^w_t$, let the expected value be given as $\expectation{\mathcal{D}^w_t}$, and variance as $\variance{\mathcal{D}^w_t}$. For a given choice of model $w$, we define the following functions:
\begin{align}
\label{eq:perModelExp}
\perModelExp{w,t} = \expectation{\mathcal{D}^w_t}, \\
\label{eq:perModelVar}
\perModelVar{w,t} = \variance{\mathcal{D}^w_t}.
\end{align}

Informally, given a model choice $w$, \cref{eq:perModelExp} returns the expected difference between output of the perception model and the true output \big(\textit{i.e.} $\expectation{\perModel{w,t}-s_t}$\big), and \cref{eq:perModelVar} returns the variance \big(\textit{i.e.} $\variance{\perModel{w,t}-s_t}$\big). Similarly, we extend the operator $\expectation{\cdot}$ and $\variance{\cdot}$ to the vector ($\perVecChoice{c} - \perVec$) in the obvious manner, denoted as $\expectation{\perVecChoice{c} - \perVec}$ and $\variance{\perVecChoice{c} - \perVec}$ respectively (\cref{appx:exp_diff}).

Naturally, in the setting described above, \cref{prob:model_selection} recasts itself from minimizing true total loss (see \cref{prob:model_selection_naive}) to minimizing expected total loss. We formally state the problem as follows:

\begin{problem}[Model Selection Problem]
\label{prob:model_selection_expectation}
Given a system with dynamics as in \cref{eq:dynamics}, and perception model distribution \big($\perModelExp{\cdot}$ and $\perModelVar{\cdot}$\big) with cost $\perModelCost{\cdot}$, compute a set of choices $\choiceVec_{opt}$ for $H$ time steps, such that, for a given $\alpha,\beta$:
\begin{equation}
\label{opt:model_selection_expectation}
\begin{aligned}
\min_{\choiceVec \in \mathbf{C}} \quad \expectation{\alpha \cdot  \Big((\perVecChoice{\choiceVec} - \perVec)^{\top} \times \Psi \times (\perVecChoice{\choiceVec} - \perVec) \Big) + \beta \cdot \perceptionCost{\choiceVec}}.
\end{aligned}
\end{equation}
\end{problem}

Given $\expectation{\perVecChoice{c} - \perVec}=\expectation{\mathcal{V}(\binChoiceVec)}$ and $\variance{\perVecChoice{c} - \perVec}=\variance{\mathcal{V}(\binChoiceVec)}$, \cref{prob:model_selection_expectation} can be solved by the following optimization problem, proposed in \cref{sol:model_selection_expectation}

\begin{solution}[Solution to \cref{prob:model_selection_expectation}]
\label{sol:model_selection_expectation}
The following Boolean optimization problem provides an optimal solution to \cref{prob:model_selection_expectation}.
\begin{equation}
\label{opt:model_selection_expectation_sol}
\begin{aligned}
\min_{\binChoiceVec} \quad & \alpha \cdot  \Big(\expectation{\mathcal{V}(\binChoiceVec})^{\top} \times \Psi \times \expectation{\mathcal{V}(\binChoiceVec)} \Big) + 
\\
\quad & \alpha \cdot \mathbb{V} \cdot \variance{\mathcal{V}(\binChoiceVec)} + 
\beta \cdot \perceptionCost{\binChoiceVec} \\
\textrm{s.t.} \quad & \forall_{t} \sum\limits_{i=0}^{W-1} b^t_i = 1\\
  &\forall_{t,i} b^t_i \in \integers_{[0,1]},
\end{aligned}
\end{equation}
where 
\begin{equation}
    \mathbb{V} =
    \begin{bmatrix}
    \Psi[0,0] & \Psi[1,1] & \hdots & \Psi[K-1,K-1]
    \end{bmatrix},
\end{equation}
where $K=p \cdot H$ is the dimension of $\Psi$.
\end{solution}

\begin{theorem}
\cref{sol:model_selection_expectation} solves \cref{prob:model_selection_expectation}, assuming the random variables $\mathcal{D}^w_t$, for all $t$ and $w$, are independent.
\end{theorem}
\begin{proof}
If the random variables are $\mathcal{D}^w_t$ (for all $t$ and $w$) are independent, the objective function given in \cref{opt:model_selection_expectation} can be rewritten as (see \cref{appx:minimizing_exp}):
\begin{align*}
    \expectation{\alpha \cdot  \Big((\perVecChoice{\choiceVec} - \perVec)^{\top} \times \Psi \times (\perVecChoice{\choiceVec} - \perVec) \Big) + \beta \cdot \perceptionCost{\choiceVec}} = \\
\alpha \cdot  \Big(\expectation{(\perVecChoice{\choiceVec} - \perVec)}^{\top} \times \Psi \times \expectation{(\perVecChoice{\choiceVec} - \perVec)} \Big) + 
\\
\alpha \cdot \mathbb{V} \cdot \variance{(\perVecChoice{\choiceVec} - \perVec)} \nonumber 
+ \beta \cdot \perceptionCost{\choiceVec}.
\end{align*}
Since $\perVecChoice{c} - \perVec=\mathcal{V}(\binChoiceVec)$, with constraints $\forall_{t} \sum\limits_{i=0}^{W-1} b^t_i = 1$ and
$\forall_{t,i} b^t_i \in \integers_{[0,1]}$, the above equation evaluates to the following:
\begin{align*}
    \Big(\expectation{\mathcal{V}(\binChoiceVec})^{\top} \times \Psi \times \expectation{\mathcal{V}(\binChoiceVec)} \Big) + \\
\alpha \cdot \mathbb{V} \cdot \variance{\mathcal{V}(\binChoiceVec)} + 
\beta \cdot \perceptionCost{\binChoiceVec}.
\end{align*}
Finally the optimization in \cref{opt:model_selection_expectation_sol} is formulated by replacing the objective function with the above equivalent function, and the added constraints. 
\end{proof}

\begin{remark}
\label{remark:np_status}
The Boolean optimization formulation in \cref{sol:model_selection_expectation}, using $H \cdot W$ Boolean variables, is NP-Hard. \cite{qp_complexity,qp_complexity_2}
\end{remark}

We note that \cref{remark:np_status} states the proposed solution is NP-Hard; it does not claim that the problem is NP-Hard. The complexity class of the model selection problem is currently unknown, and is left as a future work.

Note that though we consider a more general version in \cref{prob:model_selection_expectation} than \cref{prob:model_selection}, we still do not introduce any additional variables to the optimization formulation. The only additional term is added in the objective function, which is again linear.

\section{A Special Subcase: Model Selection in Polynomial Time Complexity}
\label{sec:polytime_subcase}
\begin{figure}[t]
\begin{center}
\includegraphics[width=\linewidth]{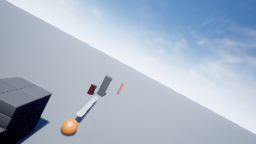}
\end{center}
\caption{\textbf{Dataset.} The dataset is gathered by equipping the drone with a camera in the given environment.}
\label{fig:airsim_dataset}
\end{figure}

\begin{figure*}
    % \begin{subfigure}{.48\textwidth}
    %     \centering
    %     \includegraphics[width=\linewidth]{images/airsim_env.pdf}
    %     \vspace{0.00em}
    %     \caption{\textbf{AirSim Environment.} The environment emulates a drone landing on a helipad, where the environment also has a lot of other fixed and moving objects.}
    %     \label{fig:airsim_env}
    % \end{subfigure}
    % \hfill
    % \begin{subfigure}{.48\textwidth}
    %     \centering
    %     \includegraphics[width=\linewidth]{images/airsim_dataset.png}
    %     %\vspace{-1.4em}
    %     \caption{\textbf{Dataset.} The dataset is gathered by equipping the drone with a camera in the given environment.}
    %     \label{fig:airsim_dataset}
    % \end{subfigure}
    % \vfill
    % \vspace{4em}
    \begin{subfigure}{.48\textwidth}
        \centering
        \includegraphics[width=\linewidth]{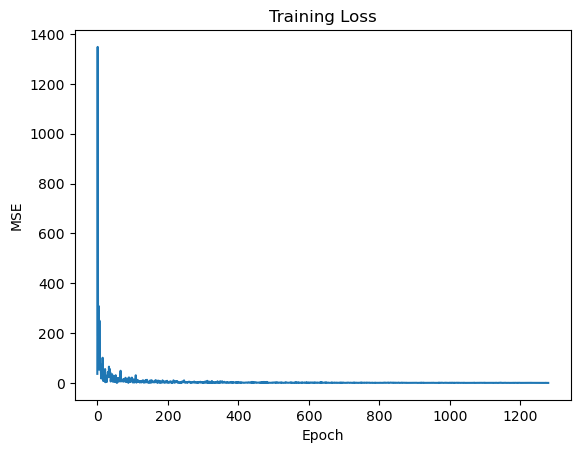}
        %\vspace{-1.4em}
        \caption{\textbf{EfficientNet b7 Training Loss.} The $x$ axis represents the epoch, and the $y$ axis represents the MSE loss at that epoch. 
        }
        \label{fig:b7_training_plot}
    \end{subfigure}
    \hfill
     \begin{subfigure}{.48\textwidth}
        \centering
        \includegraphics[width=\linewidth]{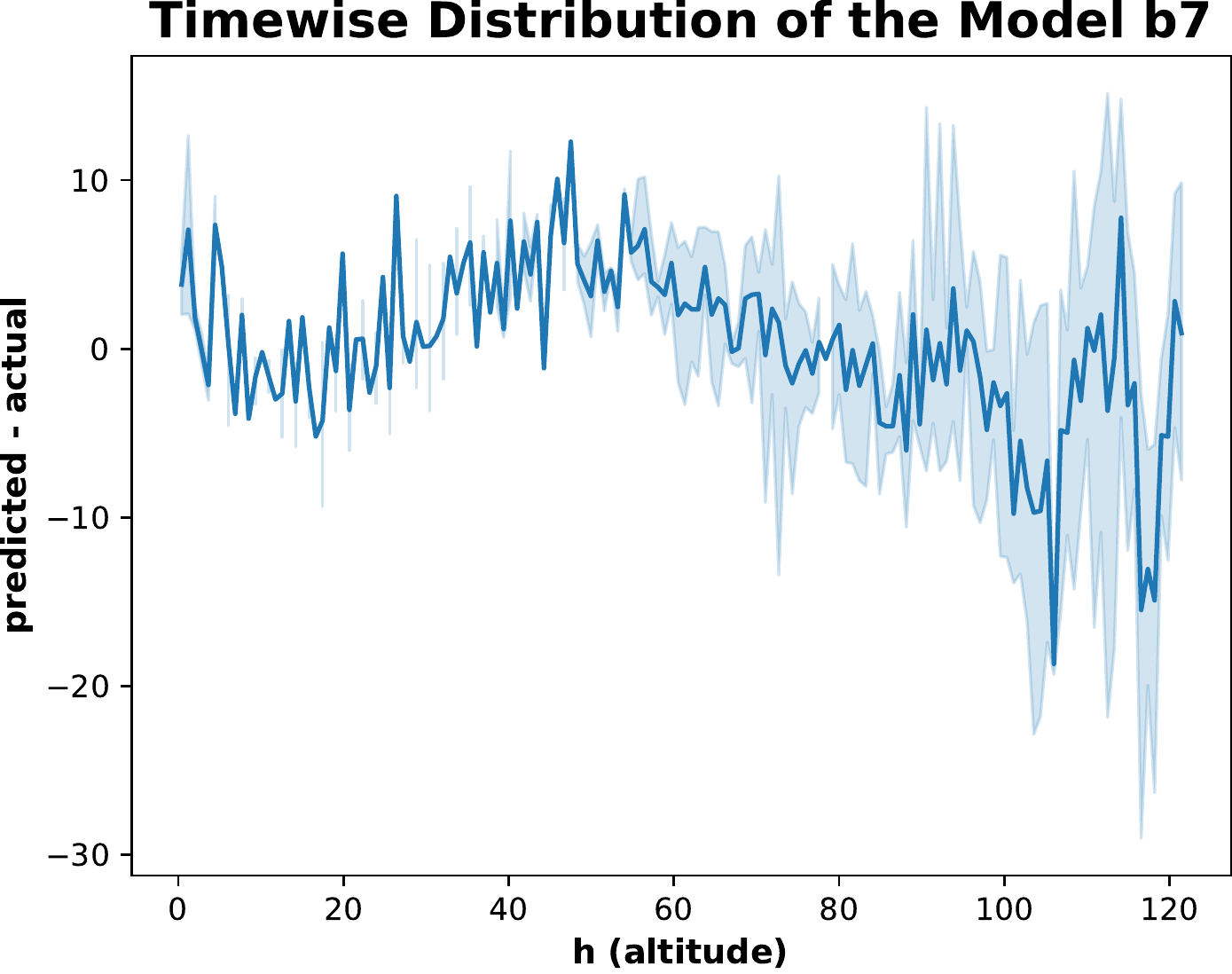}
        \vspace{0em}
        \caption{\textbf{EfficientNet b7 Distribution.} The $x$ axis represents the altitude of the drone, and the $y$ axis represents the corresponding distribution at that altitude.}
        \label{fig:airsim_distro_b7}
    \end{subfigure}
    %\caption{Computed mean values of $\devub$ (blue) and deviations (cyan) with varying $c$.}
    \caption{Data Gathering using AirSim}
    \label{fig:dataGathering}
\end{figure*}

The models that we have been considering so far were discrete, \textit{i.e.}, the choice of the model $w$ in \cref{eq:perceptionModel} and \cref{eq:perceptionCost} in \cref{sol:model_selection}, are given by integer choices $w \in \integers_{[0,W-1]}$. Similarly the probabilistic models, given in \cref{eq:perceptionModelDistro} and \cref{eq:perceptionCostDistro}, in \cref{sol:model_selection_expectation} consider discrete choices of models, \textit{i.e.}, $w \in \integers_{[0,W-1]}$. In this section, contrary to our models in  \cref{eq:perModelExp}, \cref{eq:perModelVar} and \cref{eq:perceptionCostDistro}, we consider models  in a continuous setting. That is, we consider a continuous set of models to solve \cref{prob:model_selection_expectation} in polynomial time. Formally, we consider the following models.

\begin{align}
\label{eq:perModelExpCont}
\perModelExpCont{c,t} &=& (1-c) \cdot \expectation{\mathcal{D}^0_t} + c \cdot \expectation{\mathcal{D}^{W-1}_t}, \\
\label{eq:perModelVarCont}
\perModelVarCont{c,t} &=& (1-c) \cdot \variance{\mathcal{D}^0_t} + c \cdot \variance{\mathcal{D}^{W-1}_t},
\end{align}

where $c \in \reals_{[0,1]}$,  $\mathcal{D}^0_t$ and $\mathcal{D}^{W-1}_t$ are random variables representing the least and the most accurate distributions, as mentioned before. The cost model is given as:

\begin{equation}
    \label{eq:perceptionCostDistroCont}
    \perModelCostCont{c}=c \cdot \Upsilon.
\end{equation}

Intuitively, this setting assumes a linear combination of the best and the worst possible models are available. In other words, with $c=1$, we get the best possible output \big(\textit{i.e.} $\expectation{\mathcal{D}^0_t}$ and $\variance{\mathcal{D}^0_t}$ closer to 0\big) with cost $\Upsilon$, and vice-versa with $c=0$.

Next, we restate \cref{prob:model_selection_expectation} in this setting and propose a polynomial time solution. But before we proceed to restating \cref{prob:model_selection_expectation}, we are required to put forward the necessary artifacts.

Let $\mathbf{C}^c$ be the set of all possible choices. Let $\choiceVec=\big\{c_0, c_1, \cdots, c_{H-1}\big\} \in \mathbf{C}^c$ be a set of choices up-to time step $H-1$, where $c_i \in \reals_{[0,1]}$. Using this, we rewrite $\expectation{\perVecChoice{c} - \perVec}$ by replacing the $H \cdot W$ binary variables with $H$ real variables as follows:

\begin{align}
\label{eq:perDiffEncodingCont}
&\expectation{\perVecChoice{c} - \perVec}
=
\expectation{\mathcal{V}
(\choiceVec)} 
= \nonumber
\\
&\begin{bmatrix}
(1-c_0) \cdot \expectation{\mathcal{D}^0_0} + c_0 \cdot \expectation{\mathcal{D}^{W-1}_0}\\
\vdots \\
(1-c_t) \cdot \expectation{\mathcal{D}^0_t} + c_t \cdot \expectation{\mathcal{D}^{W-1}_0}\\
\vdots \\
(1-c_{H-1}) \cdot \expectation{\mathcal{D}^0_{H-1}} + c_{H-1} \cdot \expectation{\mathcal{D}^{W-1}_{H-1}}
\end{bmatrix}.
\end{align}

Now we restate \cref{prob:model_selection_expectation} in a continuous setting as follows:

\begin{problem}[Model Selection Problem with Continuous Choices]
\label{prob:model_selection_expectation_cont}
Given a system with dynamics as in \cref{eq:dynamics}, and perception model distribution \big($\perModelExpCont{\cdot}$ and $\perModelVarCont{\cdot}$\big) with cost $\perModelCostCont{\cdot}$, compute a set of choices $\choiceVec_{opt}$ for $H$ time steps, such that, for a given $\alpha,\beta$:
\begin{equation}
\label{opt:model_selection_expectation_cont}
\begin{aligned}
\min_{\choiceVec \in \mathbf{C}} \quad \expectation{\alpha \cdot  \Big((\perVecChoice{\choiceVec} - \perVec)^{\top} \times \Psi \times (\perVecChoice{\choiceVec} - \perVec) \Big) + \beta \cdot \perceptionCost{\choiceVec}}
\end{aligned}
\end{equation}
\end{problem}

Given $\expectation{\perVecChoice{c} - \perVec}=\expectation{\mathcal{V}(\choiceVec)}$ and $\variance{\perVecChoice{c} - \perVec}=\variance{\mathcal{V}(\choiceVec)}$, \cref{prob:model_selection_expectation_cont} can be solved by the following optimization problem, proposed in \cref{sol:model_selection_expectation_cont}

\begin{solution}[Solution to \cref{prob:model_selection_expectation}]
\label{sol:model_selection_expectation_cont}
The following Boolean optimization problem provides an optimal solution to \cref{prob:model_selection_expectation_cont}.
\begin{equation}
\label{opt:model_selection_cont}
\begin{aligned}
\min_{\choiceVec \in \mathbf{C}^c} \quad & \alpha \cdot  \Big(\expectation{\mathcal{V}(\choiceVec})^{\top} \times \Psi \times \expectation{\mathcal{V}(\choiceVec)} \Big) + \alpha \cdot \mathbb{V} \cdot \variance{\mathcal{V}(\choiceVec)} \\ \quad & + \beta \cdot \perceptionCost{\choiceVec} \\
\textrm{s.t.} \quad & \forall_{t} c_t \in \reals_{[0,1]},    \\
\end{aligned}
\end{equation}
where 
\begin{equation}
    \mathbb{V} =
    \begin{bmatrix}
    \Psi[0,0] & \Psi[1,1] & \hdots & \Psi[K-1,K-1] 
    \end{bmatrix}
\end{equation}
$K=p \cdot H$ is the dimension of $\Psi$.
\end{solution}

\begin{theorem}
The quadratic optimization formulation in \cref{sol:model_selection_expectation_cont}, using $H$ real variables, can be solved in polynomial time.
\end{theorem}
\begin{proof}
The proof hinges on casting the quadratic optimization problem in \cref{sol:model_selection_expectation_cont} to a semidefinite program \cite{sdp}. The details of the proof is given in \cref{subsec:polyProof}.
\end{proof}

\subsection{Polynomial Time Proof}
\label{subsec:polyProof}
Note that each element in the vector $\expectation{\mathcal{V}
(\choiceVec)}$ is a $p$ dimensional vector. Therefore, the vector $\expectation{\mathcal{V}
(\choiceVec)}$ can be unwrapped and rewritten with equivalent choice variables 
\begin{equation}
    \choiceVec' =
    \begin{bmatrix}
    c'_0 & c'_1 & \hdots & c'_{p \cdot H -1},
    \end{bmatrix}^\top
\end{equation}

where $\forall_{t=0}^{H-1}$ $c_t=c'_{p\cdot t}=\cdots=c'_{p\cdot t -1}$.
Let $\expectation{\mathcal{V}
(\choiceVec)} = \expectation{\mathcal{V}
(\choiceVec')}$  (see \cref{appx:varTransform} for details). Further, we rewrite $\perceptionCost{\choiceVec}$ \big(\cref{eq:jp}\big) in terms of $\choiceVec'$ as $\Im\choiceVec'$; where $\Im \in \reals^{1 \times p \cdot H}$ can be computed easily $s.t.$ $\perceptionCost{\choiceVec}=\Im\choiceVec'$.

Using the above idea, we rewrite the optimization problem in  \cref{opt:model_selection_cont} as follows:
\begin{equation}
\label{eq:sdp_proof_4}
\begin{aligned}
\min_{\choiceVec \in \mathbf{C}^c} \quad & \alpha \cdot  \Big(\expectation{\mathcal{V}(\choiceVec')}^{\top} \times \Psi \times \expectation{\mathcal{V}(\choiceVec')} \Big) \\ \quad & + \alpha \cdot \mathbb{V} \cdot \variance{\mathcal{V}(\choiceVec')} + \beta \cdot \perceptionCost{\choiceVec'} \\
\textrm{s.t.} \quad & \forall_{0 \le t \le H-1} \forall_{1 \le r \le p-1}~~~~~ W_{t,r}^{\top} \choiceVec' = 0 \\
\quad & \forall_{0 \le t \le H-1} ~~~~~~~~~~ 0 \le E_t^{\top} \choiceVec' \le 1, \\
\end{aligned}
\end{equation}

where: $W_{t,r}$ is column matrix of size $pH$ with all zeros, except $W_{t,r}[pt+(r-1)]=1$ and $W_{t,r}[pt+r]=-1$; $E_{t}$ is column matrix of size $pH$ with all zeros, except $E_t[t]=1$.
Intuitively, the first constraint encodes  $\forall_{t=0}^{H-1}$ $c'_{p\cdot t}=\cdots=c'_{p\cdot t -1}$, and the second constraint ensures that choice variables are in $\reals_{[0,1]}$.

Note that for a given $t$, $\expectation{\mathcal{D}^0_t}$ and $\expectation{\mathcal{D}^{W-1}_t}$ are also a $p$ dimensional vector. Let the $r$-th element of these vectors be denoted as $\expectation{\mathcal{D}^{0}_t}[r]$ and $\expectation{\mathcal{D}^{W-1}_t}[r]$. Let $\expectation{\mathcal{D}^{0}_t}[r]=\Gamma^0_{pt+r}$ and $\expectation{\mathcal{D}^{W-1}_t}[r]=\Gamma^1_{pt+r}$. Now we can rewrite \cref{eq:perDiffEncodingCont} with $\choiceVec'$ as follows:

\begin{align}
\label{eq:perDiffEncodingContTransform}
& \expectation{\perVecChoice{c} - \perVec}
=
\expectation{\mathcal{V}
(\choiceVec)} 
= \nonumber
\\
& \begin{bmatrix}
(1-c'_0) \cdot \Gamma^0_0 + c'_0 \cdot \Gamma^1_0\\
\vdots \\
(1-c'_{pt+r}) \cdot \Gamma^0_{pt+r} + c'_t \cdot \Gamma^1_{pt+r}\\
\vdots \\
(1-c'_{pH-1}) \cdot \Gamma^0_{pH-1} + c'_{pH-1} \cdot \Gamma^1_{pH-1}
\end{bmatrix}.
\end{align}

Note that the term, $\alpha \cdot  \Big(\expectation{\mathcal{V}(\choiceVec')})^{\top} \times \Psi \times \expectation{\mathcal{V}(\choiceVec')} \Big)$, in \cref{eq:sdp_proof_4}, has quadratic, linear and constant terms in $\choiceVec'$. Following simple algebraic steps to separate the quadratic, linear and constant terms, we can restate the optimization formulation in \cref{eq:sdp_proof_4}, in its canonical form, as follows (the steps are provided in \cref{appx:canonical}):

\begin{equation}
\begin{aligned}
\min_{\choiceVec'} \quad & \choiceVec'^{\top} \Psi' \choiceVec' + (\mathcal{L} + \mathcal{R} + \beta \cdot \Im) \choiceVec' + \mathcal{K} + \mathcal{K}_2\\
\textrm{s.t.} \quad & \forall_{0 \le t \le H-1} \forall_{0 \le r \le p-1}~~~~~ W_{t,r}^{\top} \choiceVec' = 0 \\
\quad & \forall_{0 \le t \le H-1} ~~~ 0 \le E_t^{\top} \choiceVec' \le 1,  \\
\label{opt:canonical}
\end{aligned}
\end{equation}

where $\Psi'$ is a matrix with elements: $\Psi'[i][j]=\alpha \cdot \big(\Gamma^0_i\Gamma^0_j+\Gamma^0_i\Gamma^1_j+\Gamma^1_i\Gamma^0_j+\Gamma^1_i\Gamma^1_j\big) \cdot \Psi[i][j]$. $\mathcal{L}$ is the linear term (in variables $\choiceVec'$) after evaluating the term $\alpha \cdot  \Big(\expectation{\mathcal{V}(\choiceVec'})^{\top} \times \Psi \times \expectation{\mathcal{V}(\choiceVec')} \Big)$, and $\mathcal{K}$ is the constant term. Let $\alpha \cdot \mathbb{V} \cdot \variance{\mathcal{V}(\choiceVec')} = \mathcal{R} \cdot \choiceVec' + \mathcal{K}_2$. Note that, we only evaluated the quadratic terms $\Psi'$ (and not the linear and constant terms) because the complexity of a quadratic programming is dependent on the quadratic term \cite{sdp,qp_complexity,qp_complexity_2}.

\begin{theorem}
\label{thm:psi}
$\Psi'$ is a positive semidefine matrix.
\end{theorem}
\begin{proof}
Let, $\Psi'=\Phi \circ \Psi$, where $\Phi[i][j]=\Gamma^0_i\Gamma^0_j+\Gamma^0_i\Gamma^1_j+\Gamma^1_i\Gamma^0_j+\Gamma^1_i\Gamma^1_j$, and $A \circ B$ denotes the Hadamard product of the two matrices $A$ and $B$.

We note that $\Psi$ is a positive semi-definite matrix \cite{sandeep_neurips}. Further, $\Psi'$ is obtained by multiplying each element of the matrix $\Psi$ with a factor constituting of terms from from the perception model. Imposing some practical assumptions of $\Phi$ (formally stated in \cref{appx:psi}), we observe that $\Psi'$ does not differ sufficiently from $\Psi'$ to lose its semi-definite property. That is, when the factors, coming from the perception model, are multiplied with the elements of $\Psi$ it does not lose its semi-definite property. 

In other words, imposing some practical assumptions on $\Phi$ (formally stated in \cref{appx:psi}), using perturbation theory, we get $\Phi$ to be a positive semi-definite matrix. Therefore, using \emph{Schur Product Theorem} {\cite[ p. 479, Theorem 7.5.3]{spt}}, if $\Phi$ and $\Psi$ are positive semi-definite matrices, $\Psi'=\Phi \circ \Psi$ is also positive semi-definite.

\end{proof}

Since $\Psi'$ is a positive semidefinite matrix (from \cref{thm:psi}), the optimization problem \big(as in \cref{opt:canonical}\big)
cannot be solved in (weakly) polynomial time~\cite{KOZLOV1980223}(using the well-known ellipsoid method), and thus we propose a different technique to solve the optimization problem in polynomial time in the rest of the section.

Let $\Psi'=M^\top M$ [from \cref{thm:psi}]. Next, we convert the quadratic programming in \cref{opt:canonical} to a semidefinite programming \cite{sdp_lecture_notes}, therefore be able to solve it in polynomial time.

We rewrite the optimization problem in \cref{opt:canonical} as follows (Detailed steps in \cref{appx:sdp}):

\begin{equation}
\label{opt:sdp}
\begin{aligned}
\min_{\choiceVec',\theta} \quad & \theta \\
\textrm{s.t.} \quad & \forall_{0 \le t \le H-1} \forall_{1 \le r \le p-1}~~~~~ \begin{bmatrix}
1 & 0 \\
0 & -W^{\top}_{t,r} \choiceVec'
\end{bmatrix} \succeq 0 \\
\quad & \forall_{0 \le t \le H-1} \forall_{1 \le r \le p-1}~~~~~\begin{bmatrix}
1 & 0 \\
0 & W^{\top}_{t,r} \choiceVec'
\end{bmatrix} \succeq 0 \\
\quad & \forall_{0 \le i \le pH}~~ \begin{bmatrix}
1 & 0 \\
0 & -1 - E^{\top}_{t} \choiceVec'
\end{bmatrix} \succeq 0
\\
\quad & \forall_{0 \le i \le pH}~~ \begin{bmatrix}
I & 0 \\
0 & E^{\top}_{t} \choiceVec'
\end{bmatrix} \succeq 0 \\
\quad & \begin{bmatrix}
I & M \choiceVec' \\
\choiceVec'^{\top} M^{\top}& -(\mathcal{K}+ \mathcal{K}_2) - (\mathcal{L}+ \mathcal{R} +\beta \cdot \Im) \choiceVec' + \theta
\end{bmatrix} \succeq 0,
\end{aligned}
\end{equation}

where, $R \succeq 0$ means $R$ is a positive semidefinite matrix. 

\textit{The above semidefinite programming, \cref{opt:sdp}, can be solved in polynomial time (worst case)} \cite{sdp}. This concludes our polynomial time proof.

% \section{A Non-convex Subcase: Yet Solvable in Polynomial Time Complexity}
% \label{sec:non_con_polytime_subcase}
% \input{non_con_polytime_subcase}

\section{Experiments}
\label{sec:exp}
We evaluate our approach on a photo-realistic drone landing model using visual navigation, where the state variables are altitude, rate of change of altitude, pitch angle, and the rate of change of pitch angle, with the input being the elevator angle. The control task is to land the aircraft using LQR control as in \cite{sandeep_neurips}. Note that the state of drone obtained from the mathematical model is almost never accurate, which could be due to various uncertainties, such as wind, the motion of the landing pad (say in a water-body). Note that the difference in state, obtained from the mathematical model and reality, is encoded by the $C s_t$ part of the dynamics as in \cref{eq:dynamics}. To be able to estimate the true state of the drone, one can use various perception models. 
In the following set of experiments, we will evaluate our proposed solution to the model selection problem. That is, a choice of perception model to be invoked, at a given step, obtained using our proposed approach strikes an optimal balance between the control cost and the perception cost, whereas other benchmark policies fail.
To this end, we will perform the following two experiments to evaluate or proposed optimal solution to the model selection problem:
\begin{enumerate}
    \item Simulate the availability of a plethora of perception models that can estimate the true altitude of the aircraft, with varying compute cost and accuracy, and evaluate our optimal model selection algorithm (\cref{prob:model_selection_expectation}).
    \item Similar to the above case, instead of simulating the perception models, we use AirSim \cite{airsim} to gather dataset of drone landing to train EfficientNet models b0 to b7 \cite{effnet} to accurately estimate the altitude of the drone. Once we train the models, similar to above experimental setup, we use our proposed optimal model selection solution to land the drone by balancing control cost and perception cost.
\end{enumerate}

We recall that the main purpose of this experiment is to evaluate our proposed solution to check if it indeed strikes the optimal balance between control cost and perception cost, against the following benchmarks:
\begin{enumerate}
    \item \tnm{All Small}: The model with the least accuracy and compute cost is used at all time steps.
    \item \tnm{All Large}: The model with the highest accuracy and compute cost is used at all time steps.
    \item \tnm{Random}: At a given time step, a random model is chosen to perform the task.
    \item \tnm{Optimal}: Our proposed model selection policy.
    \item \tnm{Oracle}: A strategy that assumes the results of all the model is known a priori, even without invoking the model. Note that this is an infeasible policy, we use this policy just to benchmark our proposed policy.
\end{enumerate}

Using our model selection policy, we want to reduce the total control cost for landing the drone with minimal perception cost. 
%Intuitively speaking, we observed in our experiments that a high control cost leads to a hasty, jerky landing, whereas a low control cost landing results in rather slow, comfortable landing. Our model selection strategy, on the other hand, results in a faster landing of the aircraft without compromising on the passenger comfort (\textit{i.e} no jerky landing). 

\begin{figure}[t]
\vskip 0.22in
\begin{center}
{\includegraphics[width=1.0\columnwidth]{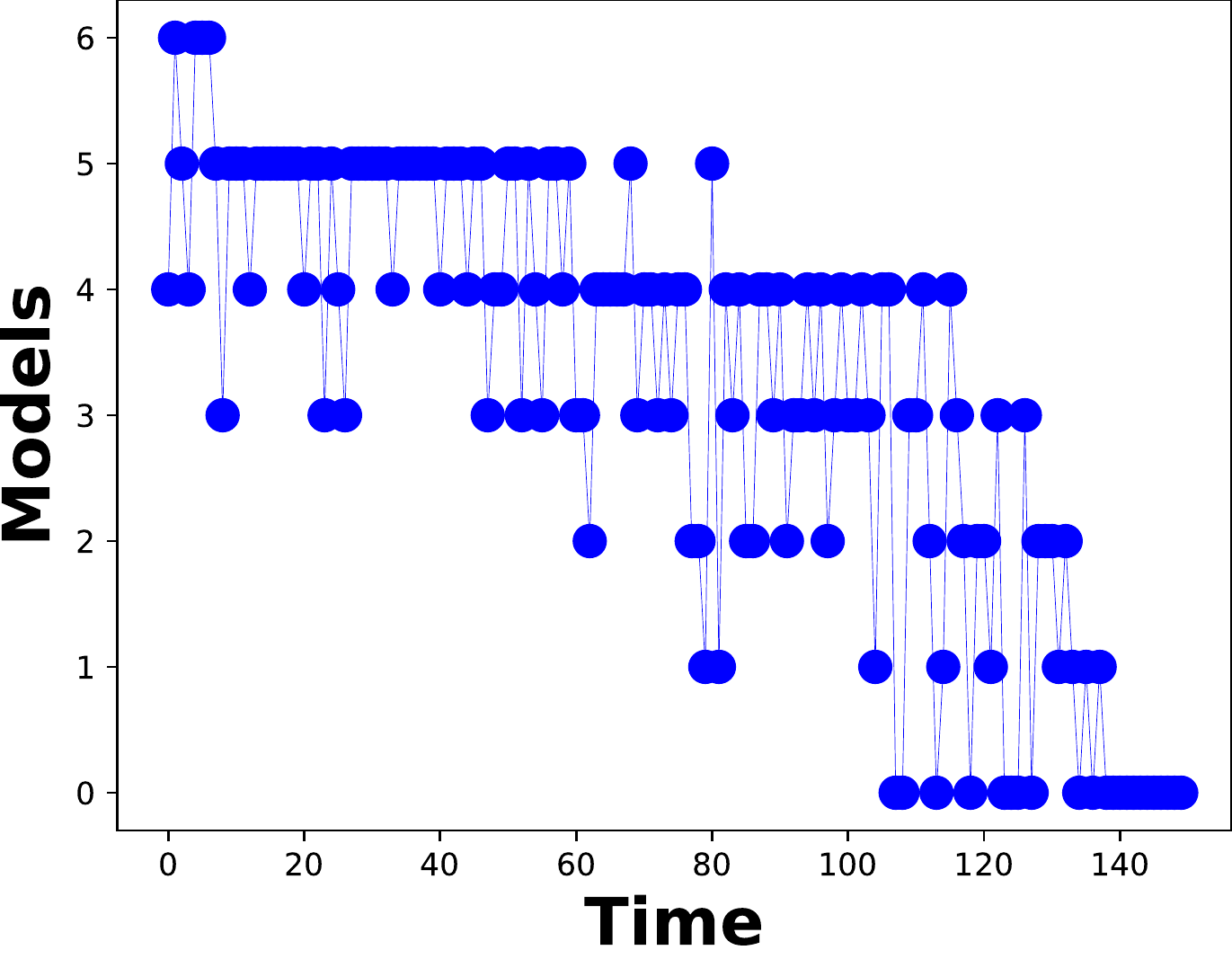}}
\end{center}
\caption{\textbf{Model Selection Sequence.} The $x$ and the $y$ axis represents time steps and the perception models invoked respectively. That is, this plot shows the perception models to be invoked at a every time step, up-to 150 time steps.}
\label{fig:simulation_ms}
\end{figure}

\begin{figure*}
    \begin{subfigure}{.48\textwidth}
        \centering
        \includegraphics[width=\linewidth]{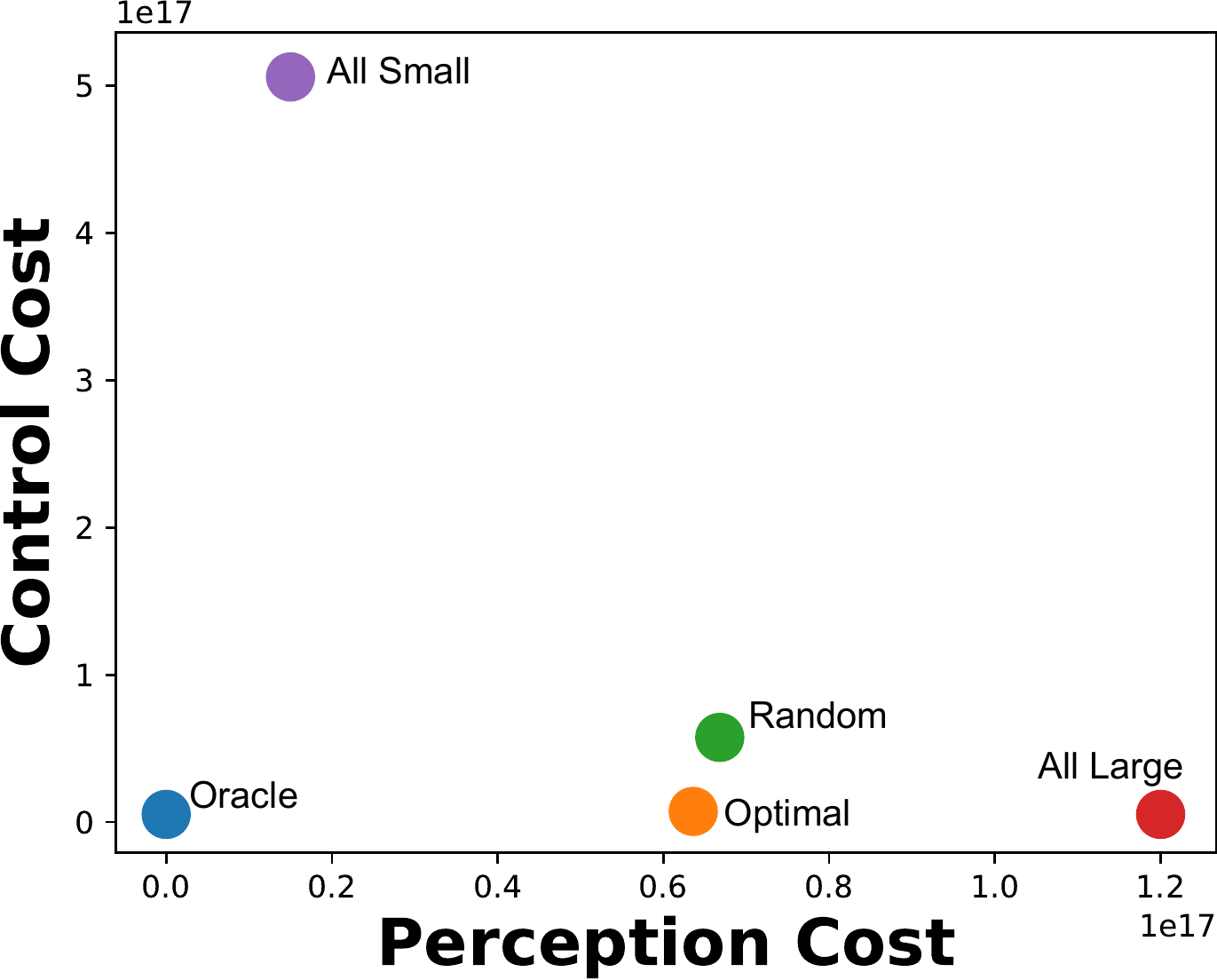}
        \vspace{0.00em}
        \caption{\textbf{Perception Model Cost vs. Control Cost.} The control cost and perception model cost for all the policies are shown. Clearly, our proposed policy (\tnm{Optimal}) achieves an optimal balance between control cost and perception model cost compared to other policies.}
        \label{fig:simulation_CvL}
    \end{subfigure}
    \hfill
    \begin{subfigure}{.48\textwidth}
        \centering
        \includegraphics[width=\linewidth]{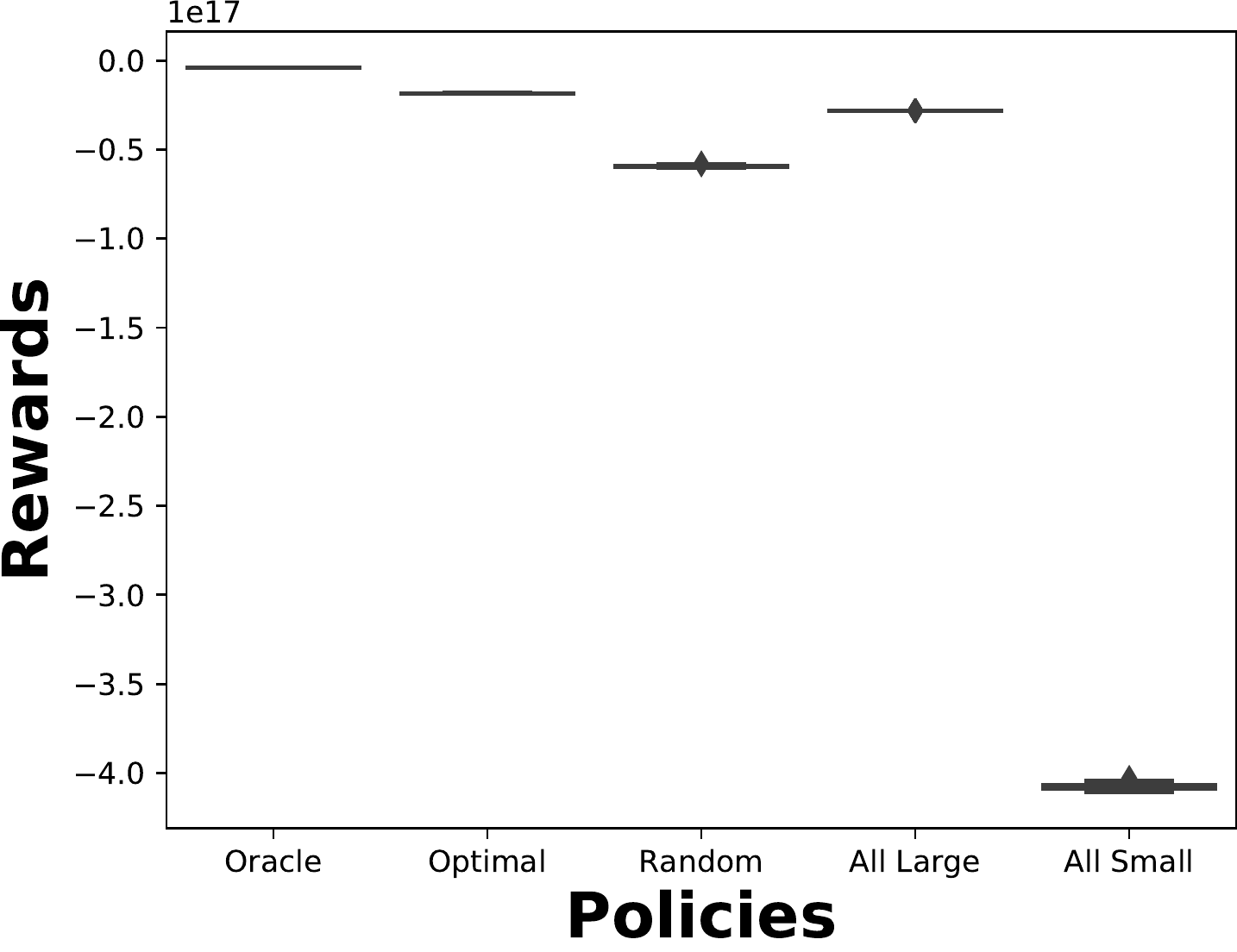}
        %\vspace{-1.4em}
        \caption{\textbf{Rewards.} Total reward obtained by various policies. Clearly, our proposed strategy (\tnm{Optimal}) obtains highest reward, and very close to the infeasible \tnm{Oracle} policy.}
        \label{fig:simulation_rwd}
    \end{subfigure}
    \vfill
    \vspace{4em}
    \begin{subfigure}{.48\textwidth}
        \centering
        \includegraphics[width=\linewidth]{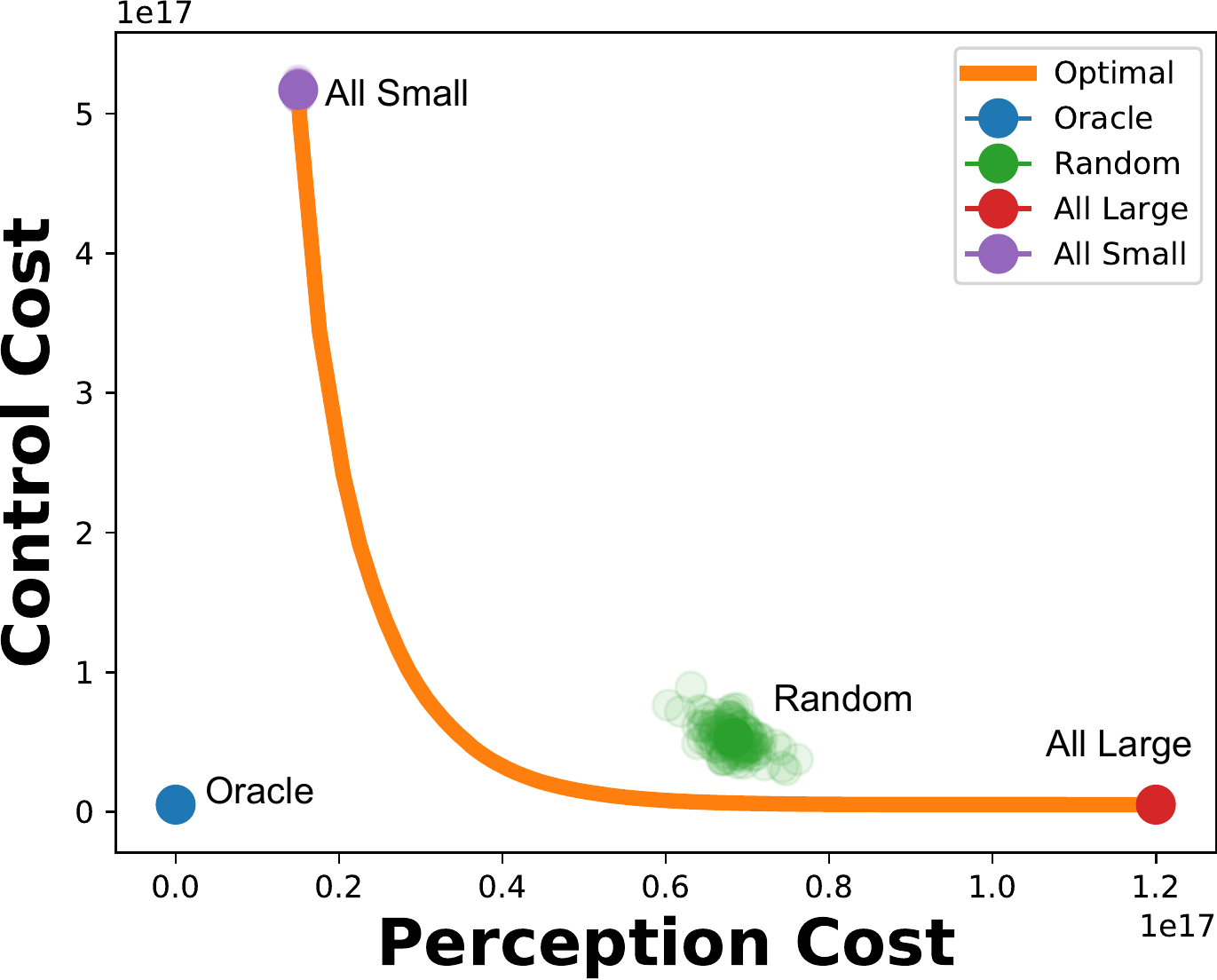}
        %\vspace{-1.4em}
        \caption{\textbf{Pareto.} This plot shows how the behavior of our optimal policy changes with varying values of $\alpha$ and $\beta$. Recall, higher values of $\alpha$ implies the model selection sequence should prioritize control cost over perception and model cost. Whereas higher values of $\beta$ implies the model selection sequence should prioritize perception model cost over control cost.}
        \label{fig:simulation_pareto}
    \end{subfigure}
    \hfill
     \begin{subfigure}{.48\textwidth}
        \centering
        \includegraphics[width=\linewidth]{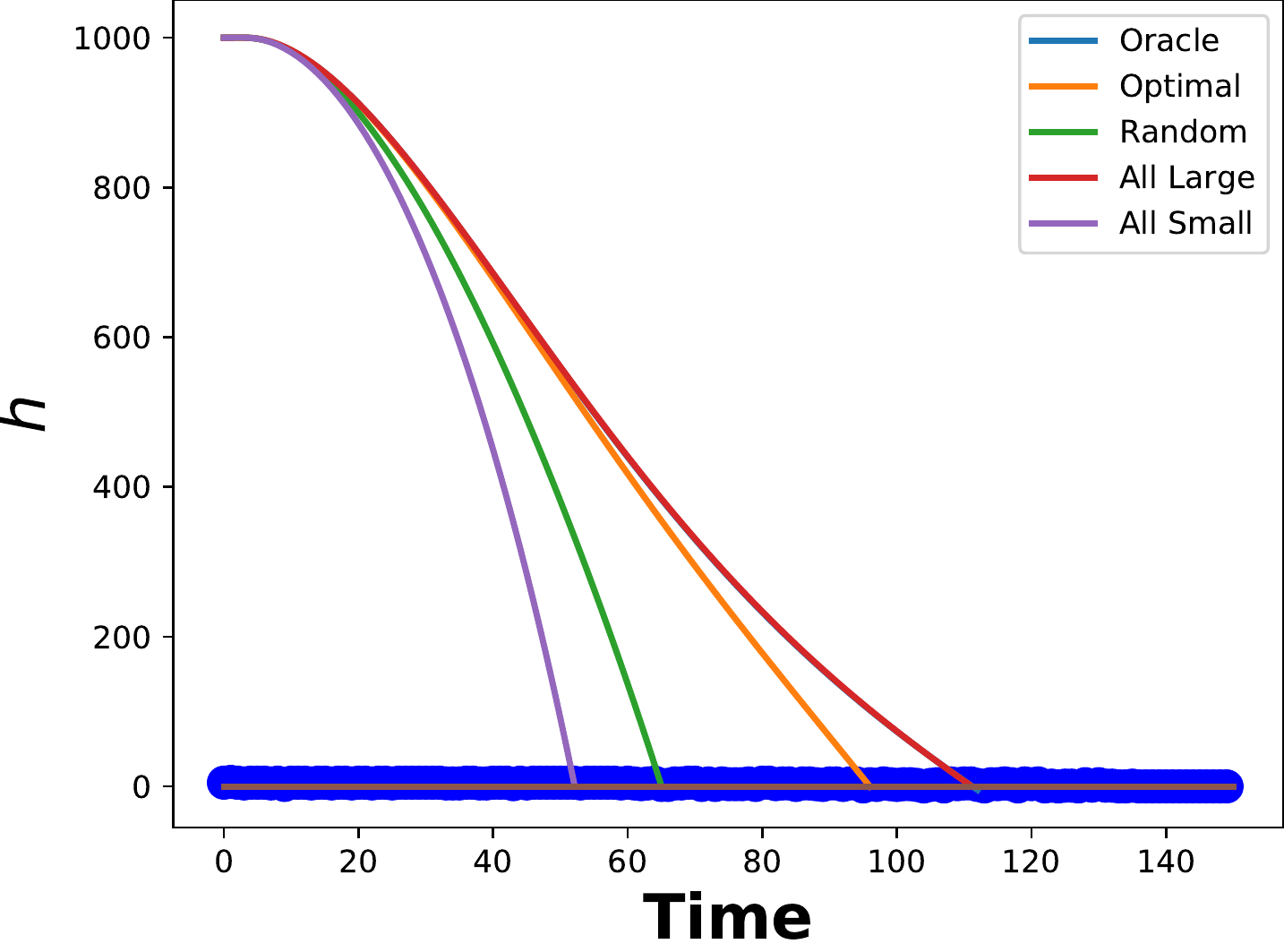}
        \vspace{0em}
        \caption{\textbf{Landing Trajectories.} Plot showing the change altitude happening, for landing the aircraft, with time for various policies.}
        \label{fig:simulation_traj_h}
    \end{subfigure}
    %\caption{Computed mean values of $\devub$ (blue) and deviations (cyan) with varying $c$.}
    \caption{Results evaluating our proposed model selection strategy (\tnm{Optimal}) on an aircraft landing simulation.}
    \label{fig:allResultSimulation}
\end{figure*}

\begin{figure}[t]
\begin{center}
{\includegraphics[width=1.0\columnwidth]{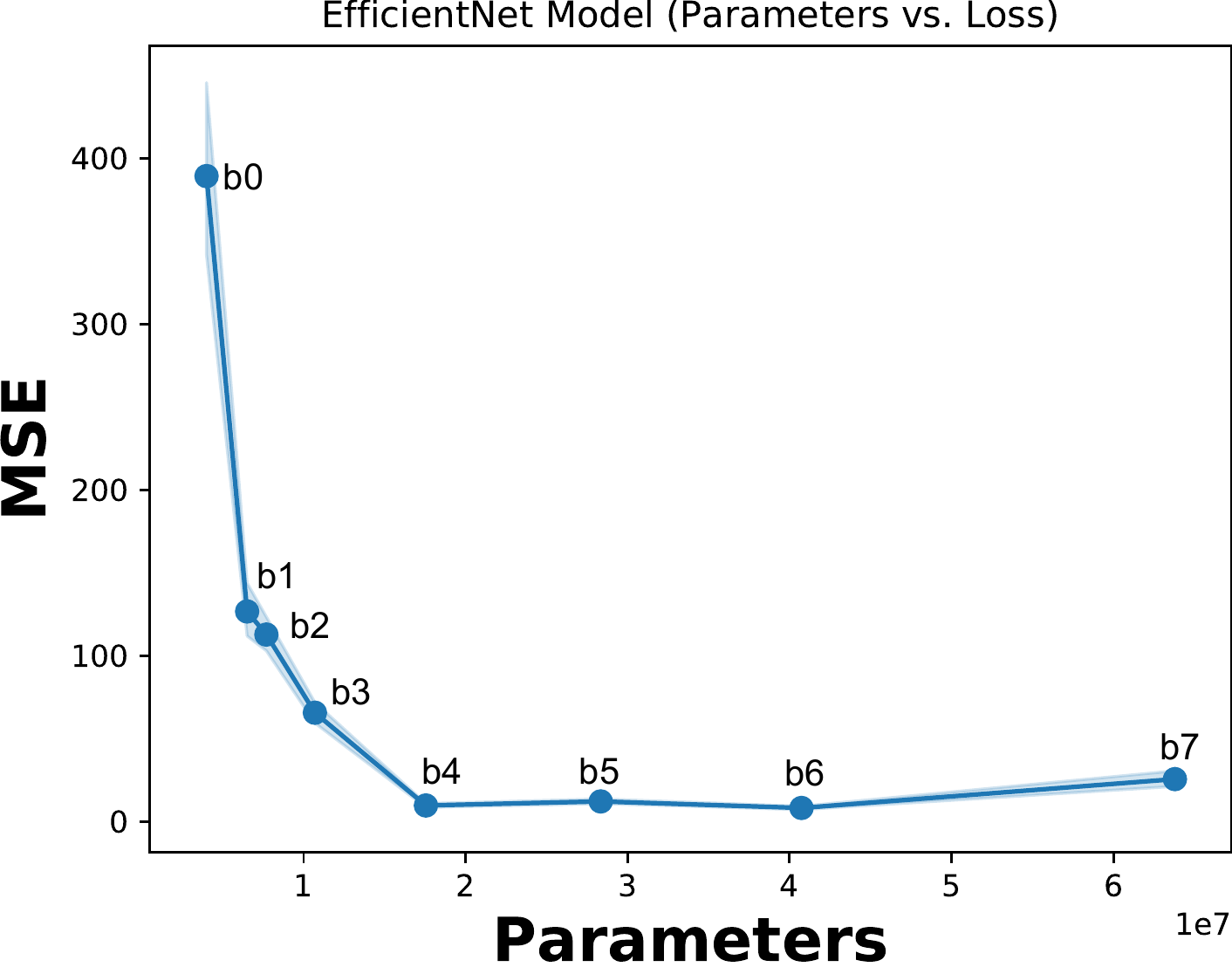}}
\end{center}
\caption{\textbf{Accuracy of All the Trained EfficientNet Models.} The $x$ and the $y$ axis represents number of parameters in the DNN and the MSE accuracy respectively. This plot shows the average accuracy and 95\% confidence interval of all the trained EfficientNet models.}
\label{fig:all_mod_acc}
\end{figure}

\begin{figure*}
    \begin{subfigure}{.48\textwidth}
        \centering
        \includegraphics[width=\linewidth]{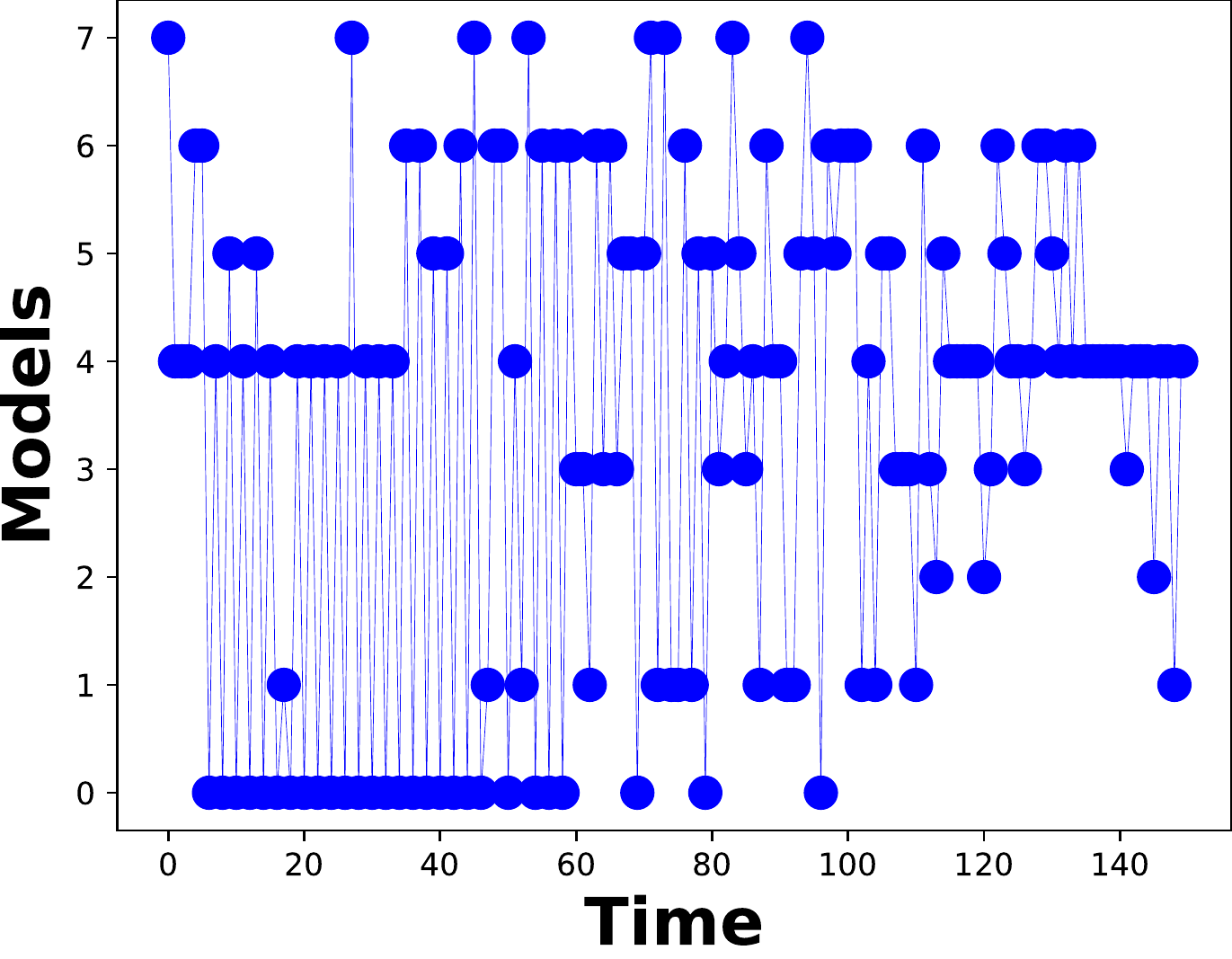}
        \vspace{0.00em}
        \caption{\textbf{Model Selection Sequence.} The $x$ and the $y$ axis represents time steps and the perception models invoked respectively. That is, this plot shows the perception models to be invoked at a every time step, up-to 150 time steps.}
        \label{fig:airsim_ms}
    \end{subfigure}
    \hfill
    \begin{subfigure}{.48\textwidth}
        \centering
        \includegraphics[width=\linewidth]{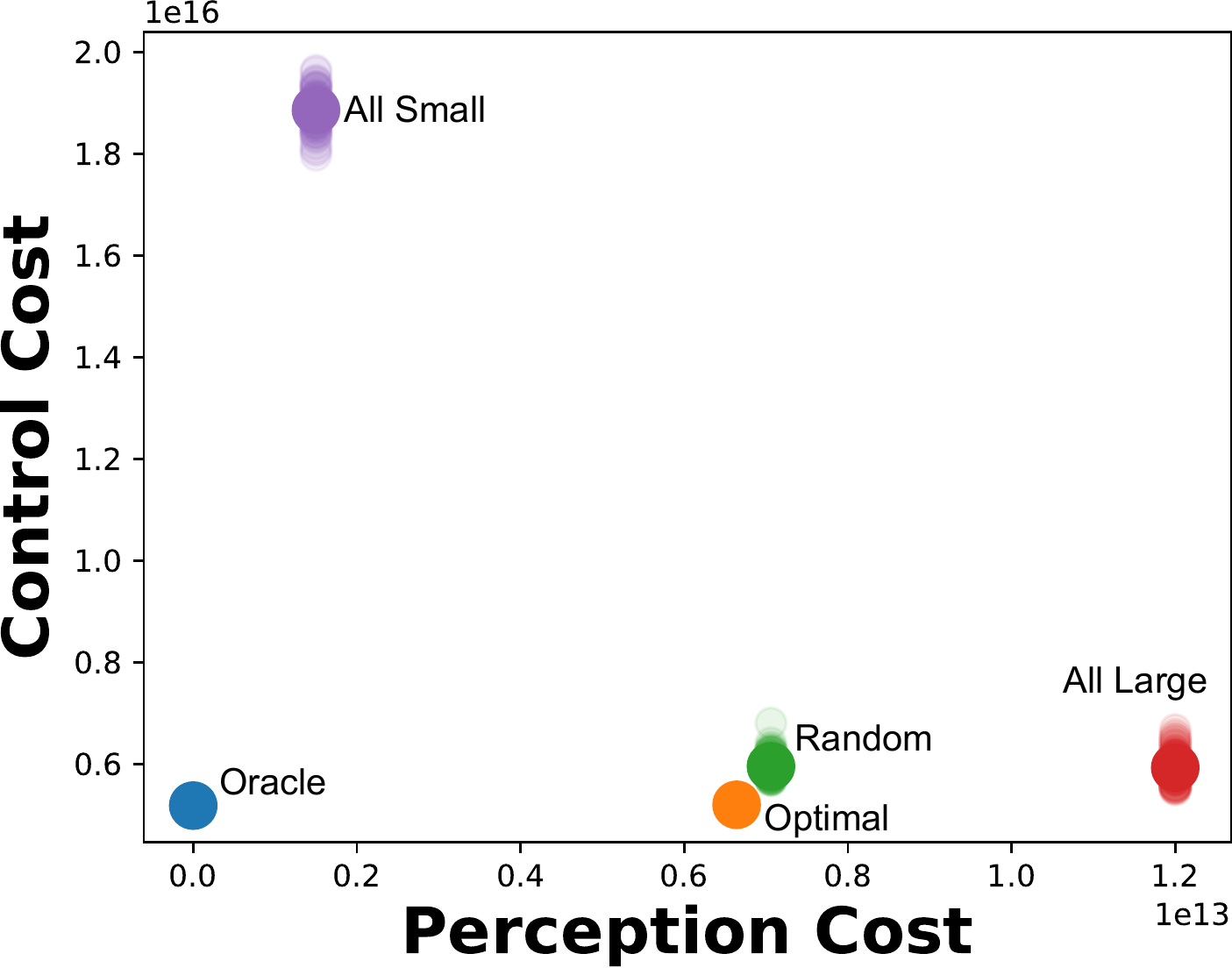}
        %\vspace{-1.4em}
        \caption{\textbf{Perception Model Cost vs. Control Cost.} The control cost and perception model cost for all the policies are shown. Clearly, our proposed policy (\tnm{Optimal}) achieves an optimal balance between control cost and perception model cost compared to other policies.}
        \label{fig:airsim_CvL}
    \end{subfigure}
    \vfill
    \vspace{4em}
    \begin{subfigure}{.48\textwidth}
        \centering
        \includegraphics[width=\linewidth]{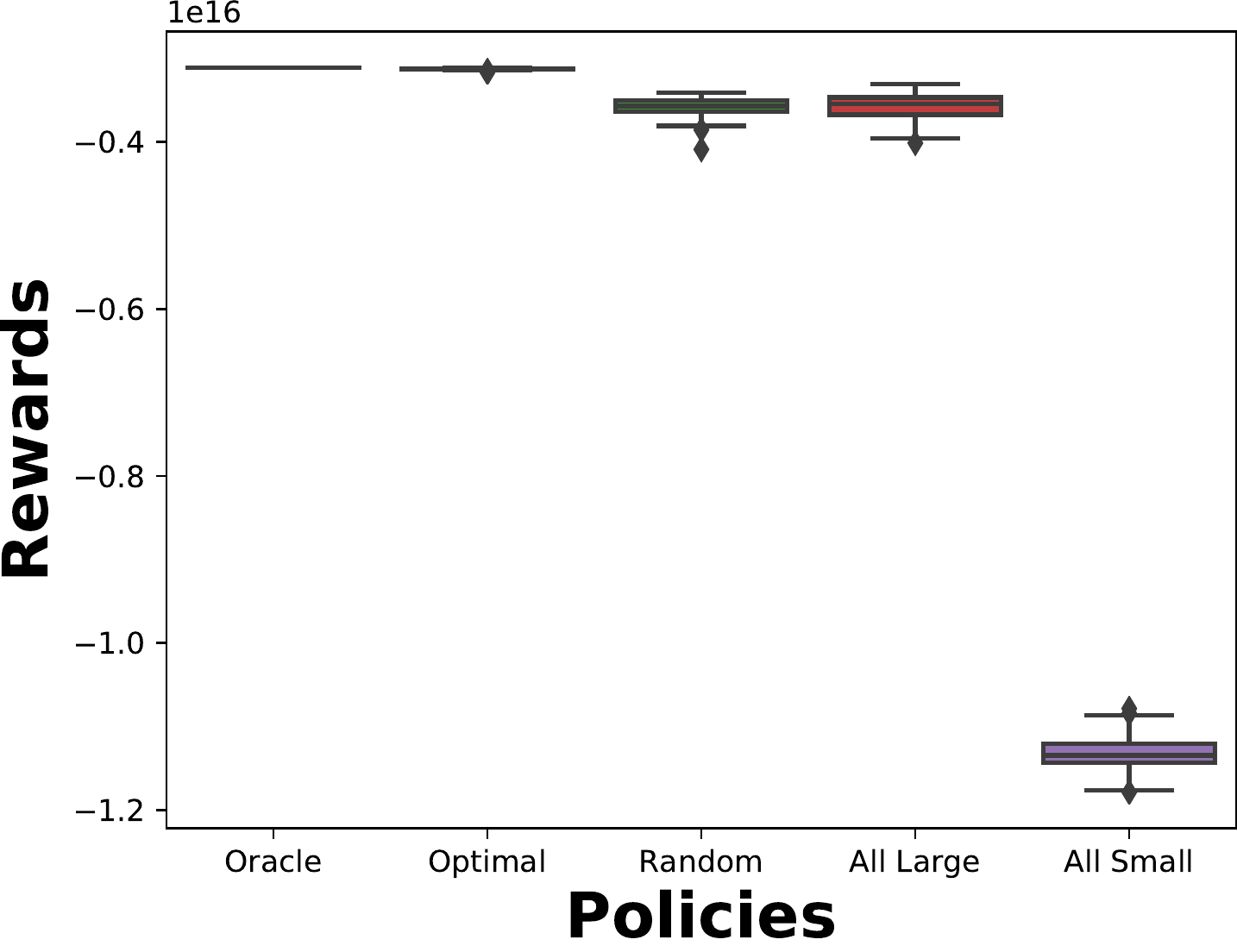}
        %\vspace{-1.4em}
        \caption{\textbf{Rewards.} Total reward obtained by various policies. Clearly, our proposed strategy (\tnm{Optimal}) obtains highest reward, and very close to the infeasible \tnm{Oracle} policy. 
        }
        \label{fig:airsim_rwd}
    \end{subfigure}
    \hfill
     \begin{subfigure}{.48\textwidth}
        \centering
        \includegraphics[width=\linewidth]{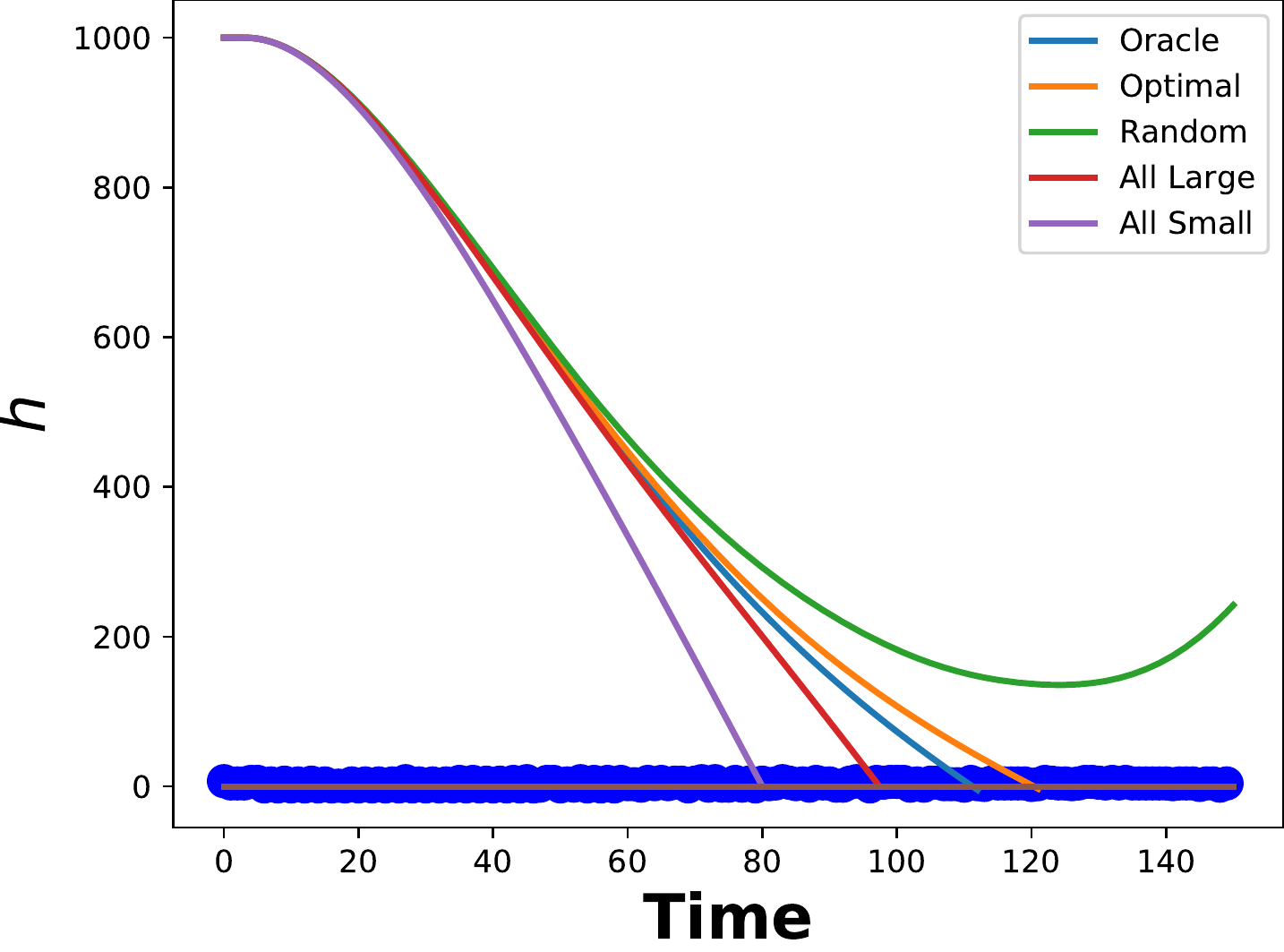}
        \vspace{0em}
        \caption{\textbf{Landing Trajectories.} Plot showing the change altitude happening, for landing the aircraft, with time for various policies.}
        \label{fig:airsim_traj_h}
    \end{subfigure}
    %\caption{Computed mean values of $\devub$ (blue) and deviations (cyan) with varying $c$.}
    \caption{Results evaluating our proposed model selection strategy (\tnm{Optimal}) on an aircraft landing simulation in AirSim.}
    \label{fig:airsim_ms_results}
\end{figure*}

\begin{figure*}[t]
\begin{center}
{\includegraphics[width=\linewidth]{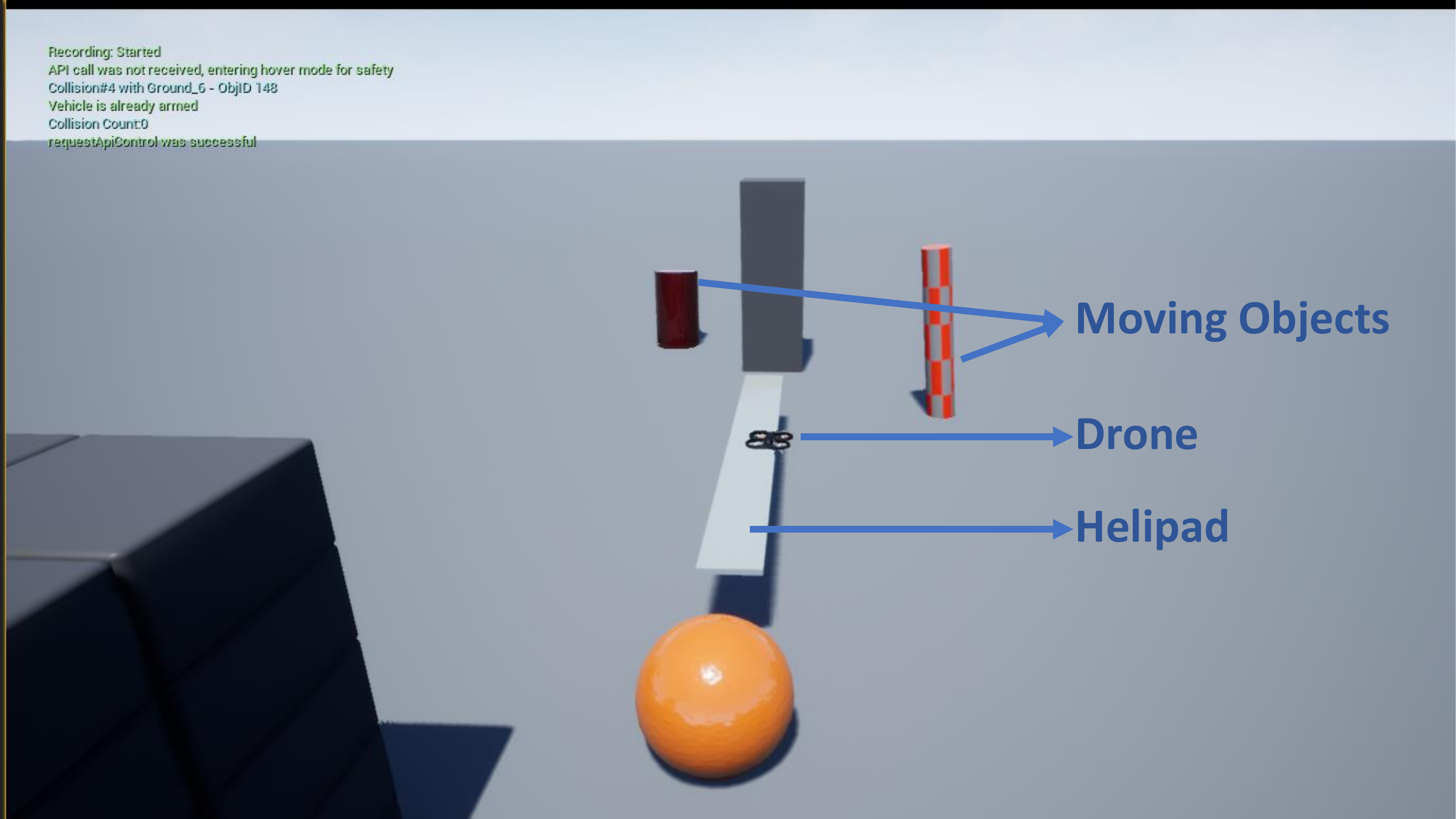}}
\end{center}
\caption{\textbf{AirSim Environment.} The environment emulates a drone landing on a helipad, where the environment also has a lot of other fixed and moving objects.}
\label{fig:airsim_env}
\end{figure*}

\subsection{Drone Landing Simulation}
\label{subsec:expSimulation}
In this subsection, we solve \cref{prob:model_selection_expectation} on an drone landing simulation. Note that our solution requires the mean and the variance of the perception models, that estimate the measurement error in the altitude, to be given as input. As a first step solution, we assume that such distributions are known to us. In particular, we assume 8 perception models are available to us, namely b0 - b7. We fix distributions, as in \cref{eq:perceptionModelDistro}, to some reasonable bounds (however, we train real perception models, and use its bounds, in the subsequent subsections). Once the required information is gathered, we solve \cref{prob:model_selection_expectation}, and obtain the model selection sequence given in \cref{fig:simulation_ms}, for 150 time steps, taking 2.88 s. Note that, in this instance, model b7 was never invoked. We used $\alpha=0.8$, $\beta=0.2$ and $\Upsilon=1e14$ (see \cref{eq:perceptionCostDistro}) to obtain these results. 

Now recall that this model selection sequence maximizes the reward in expectation. Therefore, to evaluate the obtained model selection sequence, we instantiate the assumed distribution, for all the perception models, for 100 trials. In other words, we assign concrete values to the measurement errors according to the assumed distribution. Once concrete values to the measurement errors are assigned, for all perception models, we evaluate our model selection sequence for the following.

\begin{itemize}
    \item Total perception cost and total control cost for all the polices are given in \cref{fig:simulation_CvL}. We clearly observe that our model selection strategy (\tnm{Optimal}) outperforms all other policies, achieving low control and perception cost. In particular, using our policy one can achieve 567.18\% lower control cost in 80.72\% less perception cost than the \tnm{Random} policy.
    \item Total reward gathered by all the polices is given in \cref{fig:simulation_rwd}. Our policy achieves 53.26\% better reward than the second best policy (\tnm{Cloud}), and 77\% far from the infeasible \tnm{Oracle} policy.
    \item \cref{fig:simulation_pareto} shows the perception cost vs. control cost obtained by various policies with varying values of $\alpha$ and $\beta$. Recall, a higher value of $\alpha$ implies the model selection sequence should prioritize control cost over perception cost. Whereas higher values of $\beta$ implies the model selection sequence should prioritize perception model cost over control cost.
    \item Finally, the trajectories of landing (change in altitude with time) obtained using various policies are given in \cref{fig:simulation_traj_h}.
\end{itemize}

\subsection{Photo-realistic Drone Landing Simulation in AirSim}
\label{subsec:expAirsim}
Similar to the previous set of experiments, except for assuming the distribution of the perception models, we simulate a photo-realistic drone landing sequence in AirSim using visual navigation. Contrary to our previous set of experiments (\cref{subsec:expAirsim}), we simulate the drone landing in AirSim and gather a dataset to train EfficientNet models b0 to b7. Our dataset contains the images captured from the camera attached to the drone and the altitude of the drone with respect to the helipad. Further, we compute the distribution, as in \cref{eq:perceptionModelDistro}, of the trained EfficientNet models using a different dataset (than the one used for training). Once the distribution is computed, we perform same steps as our previous experiments to solve \cref{prob:model_selection_expectation} and the evaluate the solution.

Next, we perform the following steps to compute the distribution of the perception models using AirSim.

\begin{enumerate}
    \item We setup an aircraft landing environment using a drone simulation. The drone is expected to land on a helipad. Moreover, the environment is cluttered, there are moving objects near the helipad that might make the perception difficult. The setup environment is shown in \cref{fig:airsim_env}.
    \item Using this setup, we gathered 2093 images as our training dataset, and 897 images as our testing dataset, by equipping the drone with a camera. See \cref{fig:airsim_dataset} for an example image.
    \item We then trained EfficientNet b0 to b7 on the training dataset to predict the difference in height of the drone to the helipad. See \cref{fig:b7_training_plot} for training accuracy plot of EfficientNet b7. The average accuracy of all the trained models, along with their 95\% confidence interval, is given in \cref{fig:all_mod_acc}.
    \item Once all the models are trained, we use the testing dataset to compute the distribution, as in \cref{eq:perceptionModelDistro}, for all the models. An example distribution of b7 can be found in \cref{fig:airsim_distro_b7}.
\end{enumerate}

Once the distributions are obtained, we perform similar steps as before. We now solve \cref{prob:model_selection_expectation}, and obtain the model selection sequence given in \cref{fig:airsim_ms}, for 150 time steps, taking 4.16 s. We used $\alpha=0.6$, $\beta=0.4$ and $\Upsilon=1e10$ to obtain these results.

Similar to our previous experiments we now evaluate our model selection sequence over 100 trials, and observe the following:

\begin{itemize}
    \item Total perception cost and total control cost for all the polices are given in \cref{fig:airsim_CvL}. We clearly observe that our model selection strategy (\tnm{Optimal}) outperforms all other policies, achieving low control and perception cost. In particular, using our policy one can achieve 38.04\% lower control cost in 79.1\% less perception model cost than the \tnm{Random} policy.
    \item Total reward gathered by all the polices is given in \cref{fig:airsim_rwd}. Our policy achieves 13.97\% better reward than the second best policy (\tnm{Cloud}), and just 0.3\% far from the infeasible \tnm{Oracle} policy.
    %\item \cref{fig:airsim_pareto} shows the perception model cost vs. control cost obtained by various policies with varying values of $\alpha$ and $\beta$. Recall, higher values of $\alpha$ implies the model selection sequence should prioritize control cost over perception and model cost. Whereas higher values of $\beta$ implies the model selection sequence should prioritize perception model cost over control cost.
    \item Finally, the trajectory of landing, change in altitude with time, obtained using various policies is given in \cref{fig:airsim_traj_h}.
\end{itemize}

\textit{Limitations}: In our formulation, the dynamics of the system is considered to be linear with additive uncertainties. In our future work, we wish to extend our techniques to non-linear dynamics with more intricate uncertainties. We note that such an extension is not straightforward: (i) One needs to extended the closed loop control cost formulation (ii) Encode the optimization problem in a manner that is easy for the solvers to solve efficiently. Further, we have currently focused on LQR control using some analytical results on LQR from a previous work. We wish to extend our technique to more control cost formulation, taking motivation from several robotic control applications. Advancing the multiple step decision problem to swarm robotics is also on our agenda: Consider a scenario, where instead of a single robot, a network of robots each having access to the same models are required to accomplish a control task. In a such a setup, how to allocate perception models to the robots such that optimality is ensured. Finally, we also intend to deploy the synthesized model selection policy on a real robots in a factory floor.

\section{Conclusion}
\label{sec:conclusion}
With the increasing trend in training perception model suits with latency accuracy trade-off, little attention has been paid on investigating which model to invoke for multi-step control tasks in robotics (known as the model selection problem). In this work, we provide a provably optimal solution to the model selection problem using only the information about the distribution of the models' output. The key insight from our work is that the variance of the perception models' output plays an important in multi-step decision making, while most available literature largely looks at just the mean accuracy. Simulating a perception based drone landing sequence in AirSim, our proposed technique achieved 38.04\% lower control cost with 79.1\% less perception time than other competing benchmarks. We believe the results reported in this paper will help in training perception network suites focusing on other metrics (such as variance), and not just the mean accuracy. Further, these results also provide a mathematically grounded solution to facilitate the use of available plethora of perception network to accomplish a control task.

%\section{Acknowledgements}
%ackn.

%\newpage
%\clearpage

%\bibliographystyle{plainnat}
\bibliographystyle{abbrv}
\bibliography{main.bib}

\newpage
%\clearpage

\appendix
\label{sec:appendix}
\subsection{Proof of Equivalence}
\label{appx:correctness}
In this section we provide a proof of \cref{thm:equivalence}
\begin{proof}
\cref{prob:model_selection_naive} is stated as the following minimization problem:
\begin{align}
&~&
\begin{aligned}
\min_{\choiceVec \in \mathbf{C}} \quad \totalCost{\choiceVec}.
\end{aligned} \label{eq:opt_proof_0}
\end{align}

The above equation can be equivalently restated as:
\begin{align}
&~&
\begin{aligned}
\min_{\choiceVec \in \mathbf{C}} \quad \alpha \cdot \controlCost{\choiceVec} + \beta \cdot \perceptionCost{\choiceVec}.
\end{aligned} \label{eq:opt_proof_1}
\end{align}

That is, $\choiceVec_{opt} \in \mathbf{C}$ minimizing the above optimization problem is a solution to \cref{prob:model_selection}. 

Next, we show that minimization problem in \cref{opt:model_selection_main} is equivalent to the minimization problem in \cref{eq:opt_proof_1}.

Consider the following equivalent formulations:
\begin{align}
&~&
\begin{aligned}
\min_{\choiceVec \in \mathbf{C}} \quad \alpha \cdot  \Big((\perVecChoice{\choiceVec} - \perVec)^{\top} \times \Psi \times (\perVecChoice{\choiceVec} - \perVec) \Big) +  \\ 
\beta \cdot \perceptionCost{\choiceVec}
\end{aligned} \\
&\iff&
\begin{aligned}
\min_{\choiceVec \in \mathbf{C}} \quad \alpha \cdot \Big( J^c\big(\mu(\hat{\mathbf{s}}),\perVec,x_0\big) - J^c\big(\inpVecOpt,\perVec,x_0\big) \Big) + \\
\beta \cdot \perceptionCost{\choiceVec}~~~~[\because  \cref{eq:control_cost_diff}]
\end{aligned} \\
&\iff&
\begin{aligned}
\min_{\choiceVec \in \mathbf{C}} \quad \alpha \cdot \Big(\controlCost{\choiceVec}  - J^c\big(\inpVecOpt,\perVec,x_0\big) \Big) + \\ 
\beta \cdot \perceptionCost{\choiceVec}
\end{aligned} \\
&\iff&
\begin{aligned}
\min_{\choiceVec \in \mathbf{C}} \quad \underbrace{\alpha \cdot \controlCost{\choiceVec} + \beta \cdot \perceptionCost{\choiceVec}}_{\cref{eq:opt_proof_1}} - \\ 
\underbrace{\alpha \cdot J^c\big(\inpVecOpt,\perVec,x_0\big)}_{\text{Oracle control cost}}. \label{eq:opt_proof_2}
\end{aligned} 
\end{align}

The term $\alpha \cdot J^c\big(\inpVecOpt,\perVec,x_0\big)$ encoding the oracle control cost is a constant term in the objective function of the above optimization problem, as it is independent of the optimization variable $\choiceVec$. Note that the objective function of the optimization formulation, in \cref{eq:opt_proof_2}, is obtained by subtracting the constant term $\alpha \cdot J^c\big(\inpVecOpt,\perVec,x_0\big)$ from the objective function of the optimization problem in \cref{eq:opt_proof_1}. Therefore, $\choiceVec_{opt} \in \mathbf{C}$ minimizing the objective function in \cref{eq:opt_proof_1} also minimizes the objective function in \cref{eq:opt_proof_2}.
\end{proof}

\subsection{Encoding with Binary Variables}
\label{appx:bc_equivalence}

In this section we show that $\choiceVec$ and $\binChoiceVec$ are equivalent.

Given a set of choices $\choiceVec=\big\{w_0, w_1, \cdots, w_{H-1}\big\}$, one can represent $\choiceVec$ using $\binChoiceVec=\{
    b^0_0, b^0_1, \dots, b^0_{W-1}, b^1_{0}, \dots b^1_{W-1}, \dots b^{H-1}_{W-1} 
    \}$, such that, $\forall_{0 \le t \le H-1} b^t_{w_t}=1$, and $\forall_{0 \le t \le H-1} \sum\limits_{i=0}^{W-1} b_i^t=1$.

\subsection{Expectation and Variance of the Difference Vector}
\label{appx:exp_diff}

In this section we provide the expectation and variance of ($\perVecChoice{c} - \perVec$)

\begin{align}
\label{eq:perDiffEncodingExp}
&\expectation{\perVecChoice{c} - \perVec}
=
\expectation{\mathcal{V}(\binChoiceVec)}
= \nonumber \\
&\begin{bmatrix}
\sum\limits_{i=0}^{W-1} b^0_i \cdot \expectation{\perModel{i,0} - s_0}\\
\vdots \\
\sum\limits_{i=0}^{W-1} b^t_i \cdot \expectation{\perModel{i,t} - s_t}\\
\vdots \\
\sum\limits_{i=0}^{W-1} b^{H-1}_i \cdot \expectation{\perModel{i,H-1} - s_{H-1}}
\end{bmatrix}.
\end{align}

\begin{align}
\label{eq:perDiffEncodingVar}
&\variance{\perVecChoice{c} - \perVec}
= 
\variance{\mathcal{V}(\binChoiceVec)}
= \nonumber \\
&\begin{bmatrix}
\sum\limits_{i=0}^{W-1} b^0_i \cdot \variance{\perModel{i,0} - s_0}\\
\vdots \\
\sum\limits_{i=0}^{W-1} b^t_i \cdot \variance{\perModel{i,t} - s_t}\\
\vdots \\
\sum\limits_{i=0}^{W-1} b^{H-1}_i \cdot \variance{\perModel{i,H-1} - s_{H-1}}
\end{bmatrix}.
\end{align}

\subsection{Minimizing Expectation}
\label{appx:minimizing_exp}

\begin{equation}
\label{eq:diffVec}
\perVecChoice{c} - \perVec =
\begin{bmatrix}
\per{0} - s_0 \\
\per{1} - s_1 \\
\vdots \\
\per{H-1} - s_{H-1}
\end{bmatrix}.
\end{equation}

Note that each element in the vector given in \cref{eq:diffVec} is a vector of $p$ dimension. Therefore, the vector $(\perVecChoice{c} - \perVec)$ can be rewritten as follows:

\begin{equation}
\label{eq:diffVecUnwrapped}
\perVecChoice{c} - \perVec =
\begin{bmatrix}
\per{0}[0] - s_0[0] \\
\vdots \\
\per{0}[p-1] - s_0[p-1] \\
\per{1}[0] - s_1[0] \\
\vdots \\
\per{H-1}[p-1] - s_{H-1}[p-1]
\end{bmatrix},
\end{equation}

where the $r$-th ($0 \le r \le p-1$) element of the vector $\per{t}$ is given as $\per{t}[r]$. As mentioned in \cref{sec:exp_solution_model_selection}, each element in the vector \cref{eq:diffVecUnwrapped} comes from a probability distribution; let the random variable corresponding to $\per{t}[r]-s_t[r]$ be denoted as $d_{pt+r}$, for given choice of model $w$. Therefore \cref{eq:diffVecUnwrapped} can be further rewritten as

\begin{align}
\perVecChoice{c} - \perVec 
&=&
\begin{bmatrix}
\per{0}[0] - s_0[0] \\
\vdots \\
\per{0}[p-1] - s_0[p-1] \\
\per{1}[0] - s_1[0] \\
\vdots \\
\per{H-1}[p-1] - s_{H-1}[p-1]
\end{bmatrix} \\
&=&
\begin{bmatrix}
d_0 \\
d_1 \\
\vdots \\
d_{p \cdot H -1}
\end{bmatrix}
= \mathbf{d}.
\end{align}

\begin{assumption}
\label{ass:indepRV}
Given any two random variable $d_i$, $d_j$, for $0 \le i,j \le p \cdot H -1$, they are independent.
\end{assumption}

\begin{align}
\expectation{\alpha \cdot  \Big((\perVecChoice{\choiceVec} - \perVec)^{\top} \times \Psi \times (\perVecChoice{\choiceVec} - \perVec) \Big) + \beta \cdot \perceptionCost{\choiceVec}} \\
= \expectation{\alpha \cdot  \Big(\mathbf{d}^{\top} \times \Psi \times \mathbf{d} \Big)} + \beta \cdot \perceptionCost{\choiceVec} \\
= \alpha \cdot \sum\limits_{j=0}^{p \cdot H -1} \expectation{\sum\limits_{i=0}^{p \cdot H -1} d_j \cdot d_i \cdot \Psi[i][j]} + \beta \cdot \perceptionCost{\choiceVec} \label{eq:expTmp1}
\end{align}

For any $j$ in the first summation, using \cref{ass:indepRV}, we have:

\begin{align}
\expectation{\sum\limits_{i=0}^{p \cdot H -1} d_j \cdot d_i \cdot \Psi[i][j]} = \nonumber \\
\sum\limits_{i=0}^{p \cdot H -1} \expectation{d_j} \expectation{d_i} + \variance{d_j} \cdot \Psi[j][j]. \label{eq:expTmp2}
\end{align}

Using \cref{eq:expTmp2}, we rewrite \cref{eq:expTmp1} as follows:

\begin{align}
\alpha \cdot \sum\limits_{j=0}^{p \cdot H -1} \expectation{\sum\limits_{i=0}^{p \cdot H -1} d_j \cdot d_i \cdot \Psi[i][j]} + \beta \cdot \perceptionCost{\choiceVec} \\
= \sum\limits_{j}\sum\limits_{i} \Psi[i][j] \cdot \expectation{d_j} \cdot \expectation{d_i} + \nonumber \\ \sum\limits_{i}\Psi[i][j] \cdot \variance{d_j} + \nonumber \\
 \beta \cdot \perceptionCost{\choiceVec} \\
 =
\alpha \cdot  \Big(\expectation{(\perVecChoice{\choiceVec} - \perVec)}^{\top} \times \Psi \times \expectation{(\perVecChoice{\choiceVec} - \perVec)} \Big) + \nonumber \\
\alpha \cdot \variance{(\perVecChoice{\choiceVec} - \perVec)} + \nonumber \\
 \beta \cdot \perceptionCost{\choiceVec}.
\end{align}

\subsection{Transformation of Choice Variables}
\label{appx:varTransform}
Let $\mathbf{C}^c$ be the set of all possible choices. Let $\choiceVec=\big\{c_0, c_1, \cdots, c_{H-1}\big\} \in \mathbf{C}^c$ be a set of choices up-to time step $H-1$, where $c_i \in \reals_{[0,1]}$.

\begin{align}
\label{eq:perDiffEncodingContAppx}
&\expectation{\perVecChoice{c} - \perVec}
=
\expectation{\mathcal{V}
(\choiceVec)} 
= \nonumber
\\
&\begin{bmatrix}
(1-c_0) \cdot \expectation{\mathcal{D}^0_0} + c_0 \cdot \expectation{\mathcal{D}^{W-1}_0}\\
\vdots \\
(1-c_t) \cdot \expectation{\mathcal{D}^0_t} + c_t \cdot \expectation{\mathcal{D}^{W-1}_0}\\
\vdots \\
(1-c_{H-1}) \cdot \expectation{\mathcal{D}^0_{H-1}} + c_{H-1} \cdot \expectation{\mathcal{D}^{W-1}_{H-1}}
\end{bmatrix}.
\end{align}

Note that for a given $t$, $\expectation{\mathcal{D}^0_t}$ and $\expectation{\mathcal{D}^{W-1}_t}$ are also a $p$ dimensional vector. Let the $r$-th element of these vectors be denoted as $\expectation{\mathcal{D}^{0}_t}[r]$ and $\expectation{\mathcal{D}^{W-1}_t}[r]$. Let $\expectation{\mathcal{D}^{0}_t}[r]=\Gamma^0_{pt+r}$ and $\expectation{\mathcal{D}^{W-1}_t}[r]=\Gamma^1_{pt+r}$. Now we can rewrite \cref{eq:perDiffEncodingContAppx} with $\choiceVec'$ as follows:

\begin{align}
\label{eq:perDiffEncodingContTransformAppx}
& \expectation{\perVecChoice{c} - \perVec}
=
\expectation{\mathcal{V}
(\choiceVec)} 
= \nonumber
\\
& \begin{bmatrix}
(1-c_0) \cdot \Gamma^0_0 + c_0 \cdot \Gamma^1_0\\
\vdots \\
(1-c_0) \cdot \Gamma^0_{p-1} + c_0 \cdot \Gamma^1_{p-1}\\
\vdots
\\
(1-c_{H-1}) \cdot \Gamma^0_{p(H-1)} + c_{H-1} \cdot \Gamma^1_{p(H-1)}\\
\vdots \\
(1-c_{H-1}) \cdot \Gamma^0_{pH-1} + c_{H-1} \cdot \Gamma^1_{pH-1}
\end{bmatrix}.
\end{align}

We further note that a choice scalar variable $c_t$ (Or, $1-c_t$) is multiplied with the vector $\Gamma_t^k$ in \cref{eq:perDiffEncodingContTransformAppx}. Similar to the above process, we can create copies of the scalar variable $c_t$ as follows:

\begin{align}
    c_t \cdot \Gamma^k_t 
    &=&
    \begin{bmatrix}
    c_t \cdot \Gamma^k_t & c_t \cdot  \Gamma^k_{t+1} & \cdots & c_t \cdot  \Gamma^k_{t+p-1}
    \end{bmatrix}^{\top} \\
    &=&
    \begin{bmatrix}
    c'_t \cdot \Gamma^k_t & c'_{t+1} \cdot  \Gamma^k_{t+1} & \cdots & c'_{t+p-1} \cdot  \Gamma^k_{t+p-1}
    \end{bmatrix}^{\top},
\end{align}
where $c_t = c'_t = \cdots = c'_{t+p-1}$. Using this process of unwrapping of vectors $\Gamma^k_t$ and $c_t$ to scalars, we rewrite \cref{eq:perDiffEncodingContTransformAppx} as follows:

\begin{align}
& \expectation{\perVecChoice{c} - \perVec}
=
\expectation{\mathcal{V}
(\choiceVec)} 
= \nonumber
\\
& \begin{bmatrix}
(1-c'_0) \cdot \Gamma^0_0 + c'_0 \cdot \Gamma^1_0\\
\vdots \\
(1-c'_{pt+r}) \cdot \Gamma^0_{pt+r} + c'_t \cdot \Gamma^1_{pt+r}\\
\vdots \\
(1-c'_{pH-1}) \cdot \Gamma^0_{pH-1} + c'_{pH-1} \cdot \Gamma^1_{pH-1}
\end{bmatrix},
\end{align}

where $\forall_{t=0}^{H-1}$ $c'_{p\cdot t}=\cdots=c'_{p\cdot t -1}$, and $\forall_t c'_t \in \reals_{[0,1]}$.

\subsection{Canonical Form}
\label{appx:canonical}

The objective function in \cref{eq:sdp_proof_4} formulation can be rewritten as:

\begin{align}
\alpha \cdot  \Big(\expectation{\mathcal{V}(\choiceVec')}^{\top} \times \Psi \times \expectation{\mathcal{V}(\choiceVec')} \Big) + \nonumber \\ 
\alpha \cdot \mathbb{V} \cdot \variance{\mathcal{V}(\choiceVec')} + \beta \cdot \perceptionCost{\choiceVec'} &=&
\\
\alpha \cdot   \Big(\expectation{\mathcal{V}(\choiceVec')}^{\top} \times \Psi \times \expectation{\mathcal{V}(\choiceVec')} \Big) + \nonumber
\\
\mathcal{R} \cdot \choiceVec' + \mathcal{K}_2 +\beta \cdot \Im\choiceVec' &=& \\
\Bigg( \underbrace{\sum\limits_{i=0}^{pH} \sum\limits_{j=0}^{pH} \expectation{\mathcal{V}(\choiceVec')}[i] \cdot \expectation{\mathcal{V}(\choiceVec')}[j] \cdot \alpha \cdot \Psi[i][j]}_{\text{c-term}} \Bigg) \nonumber \\
+ \mathcal{R} \cdot \choiceVec' + \mathcal{K}_2 +\beta \cdot \Im\choiceVec', \label{eq:sdp_proof_5}
\end{align}

where $\alpha \cdot \mathbb{V} \cdot \variance{\mathcal{V}(\choiceVec')} = \mathcal{R} \cdot \choiceVec' + \mathcal{K}_2$, and $\perceptionCost{\choiceVec}=\Im\choiceVec'$. Note that, the c-term in \cref{eq:sdp_proof_5} has quadratic, linear and constant terms in variables $\choiceVec'$. Therefore, we rewrite \cref{eq:sdp_proof_5} as follows:

\begin{align}
\Bigg( \sum\limits_{i=0}^{pH} \sum\limits_{j=0}^{pH} c'_i \cdot c'_j \cdot \alpha \cdot \big(\Gamma^0_i\Gamma^0_j+\Gamma^0_i\Gamma^1_j+\Gamma^1_i\Gamma^0_j+\Gamma^1_i\Gamma^1_j\big) \cdot \Psi[i][j]\Bigg) \nonumber
\\
+ \mathcal{L} \choiceVec' + \mathcal{K} + \mathcal{R} \cdot \choiceVec' + \mathcal{K}_2 +\beta \cdot \Im\choiceVec', \label{eq:sdp_proof_6}
\end{align}

where $\mathcal{L} \choiceVec'$ is the linear term (in variables $\choiceVec'$) after evaluating the c-term in \cref{eq:sdp_proof_5}, and $\mathcal{K}$ is the constant term. Note that we only evaluated the quadratic terms (and not the linear and constant term) because the complexity of a quadratic programming is dependent on the quadratic term \cite{sdp,qp_complexity_2,qp_complexity}.

\subsection{The Quadratic Matrix is a Positive Semidefinite}
\label{appx:psi}
Let, $\Phi \in \mathbb{R}^{p\cdot H \times p\cdot H}$ be as follows:
\begin{equation}
    \Phi[i][j]=\Gamma^0_i\Gamma^0_j+\Gamma^0_i\Gamma^1_j+\Gamma^1_i\Gamma^0_j+\Gamma^1_i\Gamma^1_j.
\end{equation}

Therefore, $\Psi' = \Phi \circ \Psi$, where $A \circ B$, denotes the Hadamard product of two matrices $A$ and $B$.

Let $\delta= mean_{i,j} \big\{ \Phi[i][j]\big\}$. We define a matrix, $\Delta \in \mathbb{R}^{p \cdot H \times p \cdot H}$, of same element, \textit{i.e.}, $\forall_{i,j} \Delta[i][i]=\delta$. Therefore, we write $\Phi$ as follows:

\begin{equation}
    \Phi =
    \underbrace{\tilde{\epsilon} I + \Delta}_{\Phi_c} + 
    \underbrace{
    \begin{bmatrix}
    \epsilon_{0,0}-\tilde{\epsilon} &  \epsilon_{1,1} & \cdots & \epsilon_{0,p \cdot H - 1} \\
    \epsilon_{1,0} & \epsilon_{1,1}-\tilde{\epsilon} & \cdots & \epsilon_{1,p \cdot H - 1} \\
    \vdots & \vdots & \ddots & \vdots \\
    \epsilon_{p \cdot H-1,0} & \epsilon_{p \cdot H-1,1} & \cdots & \epsilon_{p \cdot H -1,p \cdot H -1} - \tilde{\epsilon} 
    \end{bmatrix}}_{\Phi_{\delta}},
    \label{eq:phi}
\end{equation}

where $I \in \mathbb{R}^{p \cdot H\times p \cdot H}$ is an identity matrix, $\tilde{\epsilon} > 0$, $\epsilon_{i,j}= \big(\Gamma^0_i\Gamma^0_j+\Gamma'^0_i\Gamma^1_j+\Gamma^1_i\Gamma'^0_j+\Gamma^1_i\Gamma^1_j\big) - \delta$.

Note that $\Delta$ is singular but $\Phi_c$ is non-singular; this being the main reason to include the term $\tilde{\epsilon} I$ in \cref{eq:phi}, as it will be used in our further proofs.

We assume the following on the perception model:
\begin{assumption}
\label{ass:pt}
\begin{equation}
    \frac{||\Phi_\delta||_2}{||\Phi_c||_2} \le \frac{1}{\kappa(\Phi_c)} \label{eq:ass}
\end{equation}
Where $\kappa(\Phi_c)$ is defined as the condition number of a matrix, $\kappa(\Phi_c) = ||\Phi_c||_2 \cdot ||\Phi_c^{-1}||_2$.
\end{assumption}
If the assumption in \cref{ass:pt} is satisfied, we get $\Phi$ to be a positive semi-definite matrix, from perturbation theory. 

\subsection{Casting to a Semidefinite Program}
\label{appx:sdp}

The optimization formulation given in \cref{opt:canonical} can be rewritten as:

\begin{equation}
\label{eq:sdp_form_1}
\begin{aligned}
\min_{\choiceVec',\theta} \quad & \theta \\
\textrm{s.t.} \quad & \forall_{0 \le t \le H-1} \forall_{1 \le r \le p-1}~~~~~ W_{t,r}^{\top} \choiceVec' = 0 \\
\quad & \forall_{0 \le t \le H-1} ~~~ 0 \le E_t^{\top} \choiceVec' \le 1  \\
\quad & \choiceVec'^{\top} \Psi' \choiceVec' + (\mathcal{L} + \mathcal{R} + \beta \cdot \Im) \choiceVec' + \mathcal{K} + \mathcal{K}_2 - \theta \le 0.
\end{aligned}
\end{equation}

We further rewrite the above optimization as follows:
\begin{equation}
\label{eq:sdp_form_2}
\begin{aligned}
\min_{\choiceVec',\theta} \quad & \theta \\
\textrm{s.t.} \quad & \forall_{0 \le t \le H-1} \forall_{1 \le r \le p-1}~~~~~ W_{t,r}^{\top} \choiceVec' \le 0\\
\quad & \forall_{0 \le t \le H-1} \forall_{1 \le r \le p-1}~~~~~ -W_{t,r}^{\top} \choiceVec' \le 0,\\
\quad & \forall_{0 \le t \le H-1} ~~~ E_t^{\top} \choiceVec' -1 \le 0 \\ 
\quad & \forall_{0 \le t \le H-1} ~~~ -E_t^{\top} \choiceVec' \le 0 \\
\quad & \choiceVec'^{\top} \Psi' \choiceVec' + (\mathcal{L} + \mathcal{R} + \beta \cdot \Im) \choiceVec' + \mathcal{K} + \mathcal{K}_2 - \theta \le 0. \\
\end{aligned}
\end{equation}

Let $\Psi'=M^\top M$ [$\because$ \cref{thm:psi}].

Next, we express the optimization formulation in \cref{eq:sdp_form_2} as a standard semidefinite programming \cite{sdp_lecture_notes}.

\begin{equation}
\label{optAppx:sdp}
\begin{aligned}
\min_{\choiceVec',\theta} \quad & \theta \\
\textrm{s.t.} \quad & \forall_{0 \le t \le H-1} \forall_{1 \le r \le p-1}~~~~~ \begin{bmatrix}
1 & 0 \\
0 & -W^{\top}_{t,r} \choiceVec'
\end{bmatrix} \succeq 0 \\
\quad & \forall_{0 \le t \le H-1} \forall_{1 \le r \le p-1}~~~~~\begin{bmatrix}
1 & 0 \\
0 & W^{\top}_{t,r} \choiceVec'
\end{bmatrix} \succeq 0 \\
\quad & \forall_{0 \le i \le pH}~~ \begin{bmatrix}
1 & 0 \\
0 & -1 - E^{\top}_{t} \choiceVec'
\end{bmatrix} \succeq 0
\\
\quad & \forall_{0 \le i \le pH}~~ \begin{bmatrix}
I & 0 \\
0 & E^{\top}_{t} \choiceVec'
\end{bmatrix} \succeq 0 \\
\quad & \begin{bmatrix}
I & M \choiceVec' \\
\choiceVec'^{\top} M^{\top}& -(\mathcal{K}+ \mathcal{K}_2) - (\mathcal{L}+ \mathcal{R} +\beta \cdot \Im) \choiceVec' + \theta
\end{bmatrix} \succeq 0,
\end{aligned}
\end{equation}

where, $R \succeq 0$ means $R$ is a positive semidefinite matrix.

\end{document}